\documentclass{article}

\usepackage{times}
\usepackage{graphicx} 
\usepackage{url}
\usepackage{amsmath}
\usepackage{amssymb}
\usepackage{graphicx}
\usepackage[caption=false]{subfig}
\usepackage{mdwlist}
\usepackage[bookmarks=false]{hyperref}
\usepackage{natbib}
\usepackage[lined, boxed]{algorithm2e}
\usepackage{multirow}
\usepackage{amsthm}
\usepackage{color}

\graphicspath{{figs/}}

\usepackage[accepted]{icml2014}

\newtheorem{theorem}{Theorem}

\newtheorem{proposition}[theorem]{Proposition}

\icmltitlerunning{Learning Transformations for Classification Forests}

\begin{document}

\twocolumn[
\icmltitle{Learning Transformations for Classification Forests}

\icmlauthor{Qiang Qiu}{qiang.qiu@duke.edu}
\icmladdress{Department of Electrical and Computer Engineering,\\
       Duke University, \\
       Durham, NC 27708, USA}
\icmlauthor{Guillermo Sapiro}{guillermo.sapiro@duke.edu}
\icmladdress{Department of Electrical and Computer Engineering,\\
       Department of Computer Science, \\
       Department of Biomedical Engineering,\\
       Duke University, \\
       Durham, NC 27708, USA}

\icmlkeywords{Random forest, subspace clustering, classification, low-rank transformation, nuclear norm, feature learning.}

\vskip 0.3in
]

\begin{abstract}
This work introduces a transformation-based learner model for classification forests.
The weak learner at each split node plays a crucial role in a classification tree.
We propose to optimize the splitting objective by learning a linear transformation on subspaces using nuclear norm as the optimization criteria.
The learned linear transformation restores a low-rank structure for data from the same class, and, at the same time, maximizes the separation between different classes, thereby improving the performance of the split function.
Theoretical and experimental results support the proposed framework.
\end{abstract}

\section{Introduction}

Classification Forests \cite{RF2001, RFBook} have recently shown great success for a large variety of classification tasks, such as pose estimation \cite{RF-pose}, data clustering \cite{RF-clustering}, and object recognition \cite{RF-obj}. A classification forest is an ensemble of randomized classification trees.
A classification tree is a set of hierarchically connected tree nodes, i.e., \emph{split} (internal) nodes and \emph{leaf} (terminal) nodes. Each split node is associated with a different weak learner with binary outputs (here we focus on binary trees). The splitting objective at each node is optimized using the training set. During testing, a split node evaluates each arriving data point and sends it to the left or right child based on the weak learner output.

The weak learner associated with each split node plays a crucial role in a classification tree.
An analysis of the effect of various popular weak learner models can be found in \cite{RFBook}, including decision stumps, general oriented hyperplane learner, and conic section learner.
In general, even for high-dimensional data, we usually seek for low-dimensional weak learners that separate
different classes as much as possible.

High-dimensional data often have a small intrinsic dimension.
For example, in the area of computer vision, face images of a subject \cite{9point}, \cite{Wright09}, handwritten images of a digit \cite{ocr}, and trajectories of a moving object \cite{sfm}, can all be well-approximated by a low-dimensional subspace of the high-dimensional ambient space. Thus, multiple class data often lie in a union of low-dimensional subspaces.
These theoretical low-dimensional intrinsic structures are often violated for real-world data.
For example, under the assumption of Lambertian reflectance, \cite{9point} show that face images of a subject obtained under a wide variety of lighting conditions can be accurately approximated  with a 9-dimensional linear subspace. However,  real-world face images are often captured under additional  pose variations; in addition, faces are not perfectly Lambertian, and exhibit cast shadows and specularities \cite{rpca}.

When data from the same low-dimensional subspace are arranged as columns of a single matrix,  the matrix should be approximately low-rank.  Thus, a promising way to handle corrupted underlying structures of realistic data, and as such, deviations from ideal subspaces, is to restore such low-rank structure. Recent efforts have been invested in seeking transformations such that the transformed data can be decomposed as the sum of a low-rank matrix component and a sparse error one \cite{RASL, lrsalient, TILT}.
\cite{RASL} and \cite{TILT} are proposed for image alignment (see \cite{3dalign} for the extension to multiple-classes with applications in cryo-tomograhy), and \cite{lrsalient} is discussed in the context of salient object detection. All these methods build on recent theoretical and computational advances in rank minimization.

In this paper, we present a new formulation for random forests, and propose to learn a linear discriminative transformation at each split node in each tree to improve the class separation capability of weak learners.
We optimize the data splitting objective using matrix rank, via its nuclear norm convex surrogate,  as the learning criteria.
We show that the learned discriminative transformation recovers a low-rank structure for data from the same class, and, at the same time, maximize the subspace angles between different classes.
Intuitively, the proposed method shares some of the attributes of the Linear Discriminant Analysis (LDA) method, but with a significantly different metric. Similar to LDA, our method reduces intra-class variations and increases inter-class separations to achieve improved data splitting. However, we adopt the matrix nuclear norm as the key criterion to learn a transformation, being this appropriate for data expected to be in (the union of) subspaces. As shown later, our method significantly outperforms the LDA method, as well as state-of-the-art learners in classification forests.
{The learned transformations 
help in other classification task as well, e.g., subspace based
methods \cite{lowrankT}.}

\section{Transformation Forests}

A classification forest is an ensemble of binary classification trees,
where each tree consists of hierarchically connected \emph{split} (internal) nodes and \emph{leaf} (terminal) nodes.
Each split node corresponds to a weak learner, and evaluates each arriving data point and sends it to the left
or right child based on the weak learner binary outputs.
Each leaf node stores the statistics of the data points that arrived during training.
During testing, each classification tree returns a class posterior probability for a test sample, and the forest output is often defined as the average of tree posteriors.
In this section, we introduce transformation learning at each split node to dramatically improve the class separation capability of a weak learner. Such learned transformation is virtually computationally free at the testing time.

\subsection{Learning Transformation Learners}
\label{sec:form}

Consider two-class data points $\mathbf{Y} = \{ \mathbf{y}_i\}_{i=1}^N \subseteq \mathbb{R}^d$,
with each data point $\mathbf{y}_i$ in one of the $C$ low-dimensional subspaces of $\mathbb{R}^d$, and the data arranged as columns of $\mathbf{Y}$. We assume the class labels are known beforehand for training purposes. $\mathbf{Y}_+$ and $\mathbf{Y}_-$ denote the set of points in each of the two classes respectively, points again arranged as columns of the corresponding matrix.

We propose to learn a $d \times d$ linear transformation $\mathbf{T}$,\footnote{We can also consider learning a $s \times d$ matrix, $s < d$, and simultaneously reducing the data dimension.}
\begin{align} \label{nuclear_obj}
\underset{\mathbf{T}} \arg \min ||\mathbf{T Y}_+||_* + ||\mathbf{T Y}_-||_* - ||\mathbf{T [Y_+, Y_-]}||_*, \\\nonumber
~~s.t. ||\mathbf{T}||_2 = 1,
\end{align}
where $\mathbf{[Y_+, Y_-]}$ denotes the concatenation of $\mathbf{Y_+}$ and $\mathbf{Y_-}$, and
 $||\mathbf{\cdot}||_*$ denotes the matrix nuclear norm, i.e., the sum of the singular values of a matrix. The nuclear norm is the convex envelop of the rank function over the unit ball of matrices \cite{rank-min}. As the nuclear norm  can be optimized efficiently, it is often adopted as the best convex approximation of the rank function in the literature on rank optimization (see, e.g., \cite{rpca} and \cite{Recht}). The normalization condition $||\mathbf{T}||_2 = 1$ prevents the trivial solution $\mathbf{T}=0$; however, the effects of different normalizations  is interesting and is the subject of future research. Throughout this paper we keep this particular form of the normalization which was already proven to lead to excellent results.

As shown later, such linear transformation restores a low-rank structure for data from the same class, and, at the same time, maximizes the subspace angles between classes.
In this way, we reduce the intra-class variation and introduce inter-class separations to improve the class separation capability of a weak learner.

\subsection{Theoretical Analysis}
\label{sec:nuclear}

\begin{figure*} [ht]
\centering
 \subfloat[][$\begin{bmatrix} \theta_{AB}=0.79,& \theta_{AC}=0.79, &\theta_{BC}=1.05 \end{bmatrix}$ \\ $~~~~~~\mathbf{Y}_+=\{\mathbf{A}(green), \mathbf{B}(blue)\}, \mathbf{Y}_-= \{\mathbf{C}(red)\} $.] {\label{fig:3d45O} \includegraphics[angle=0, height=0.25\textwidth, width=.47\textwidth]{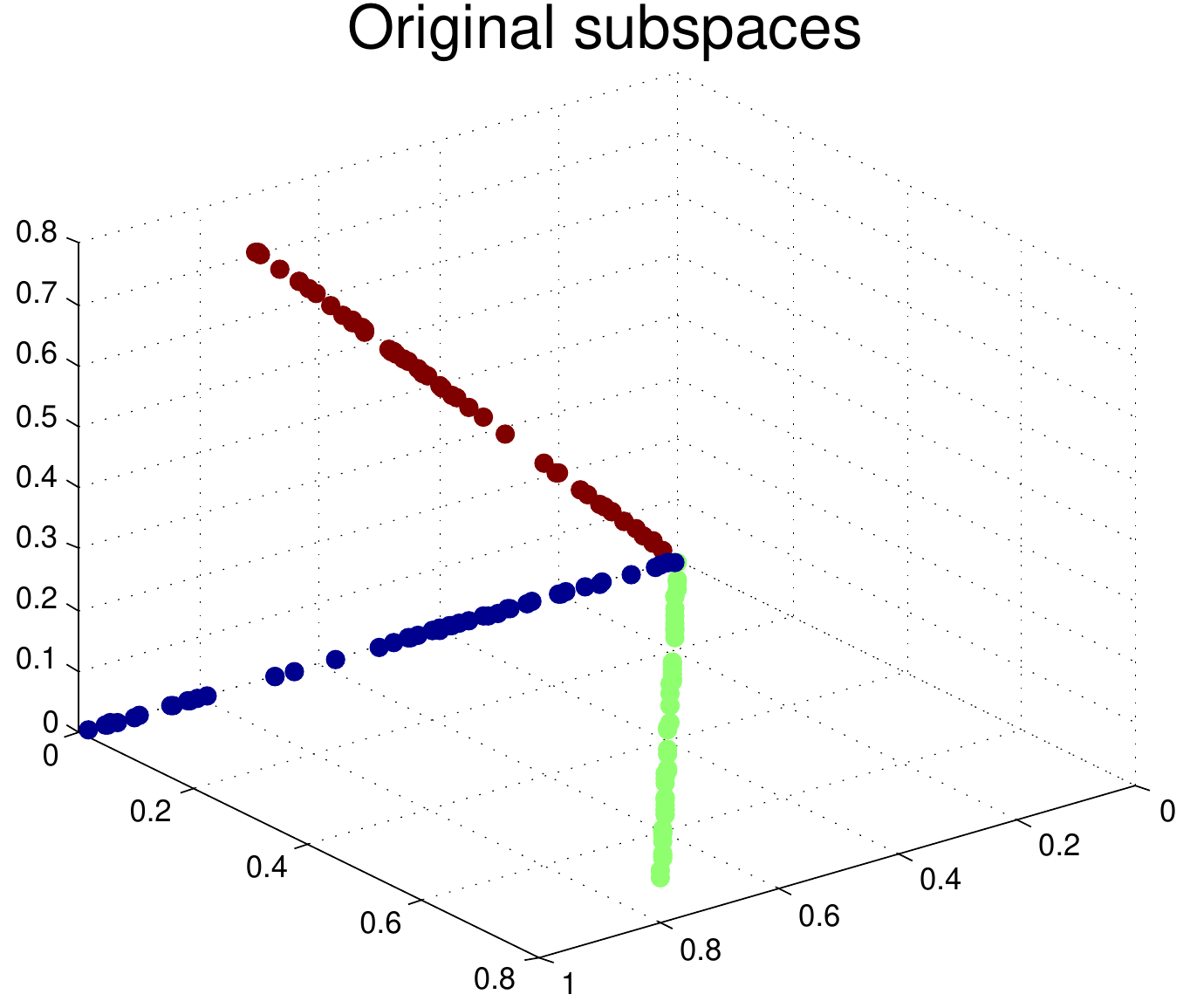}}
  \subfloat[][$\mathbf{T} = \begin{bmatrix} 0.42  & 0.33 &  -0.13 \\  0.39  &  0.32  &  -0.16 \\ -0.17  &  -0.14  &  0.81 \end{bmatrix}$; \\ $~~~~~\begin{bmatrix} \theta_{AB}= 0.006, &\theta_{AC}=1.53, &\theta_{BC}=1.53 \end{bmatrix}$.] {\label{fig:3d45T} \includegraphics[angle=0, height=0.25\textwidth, width=.45\textwidth]{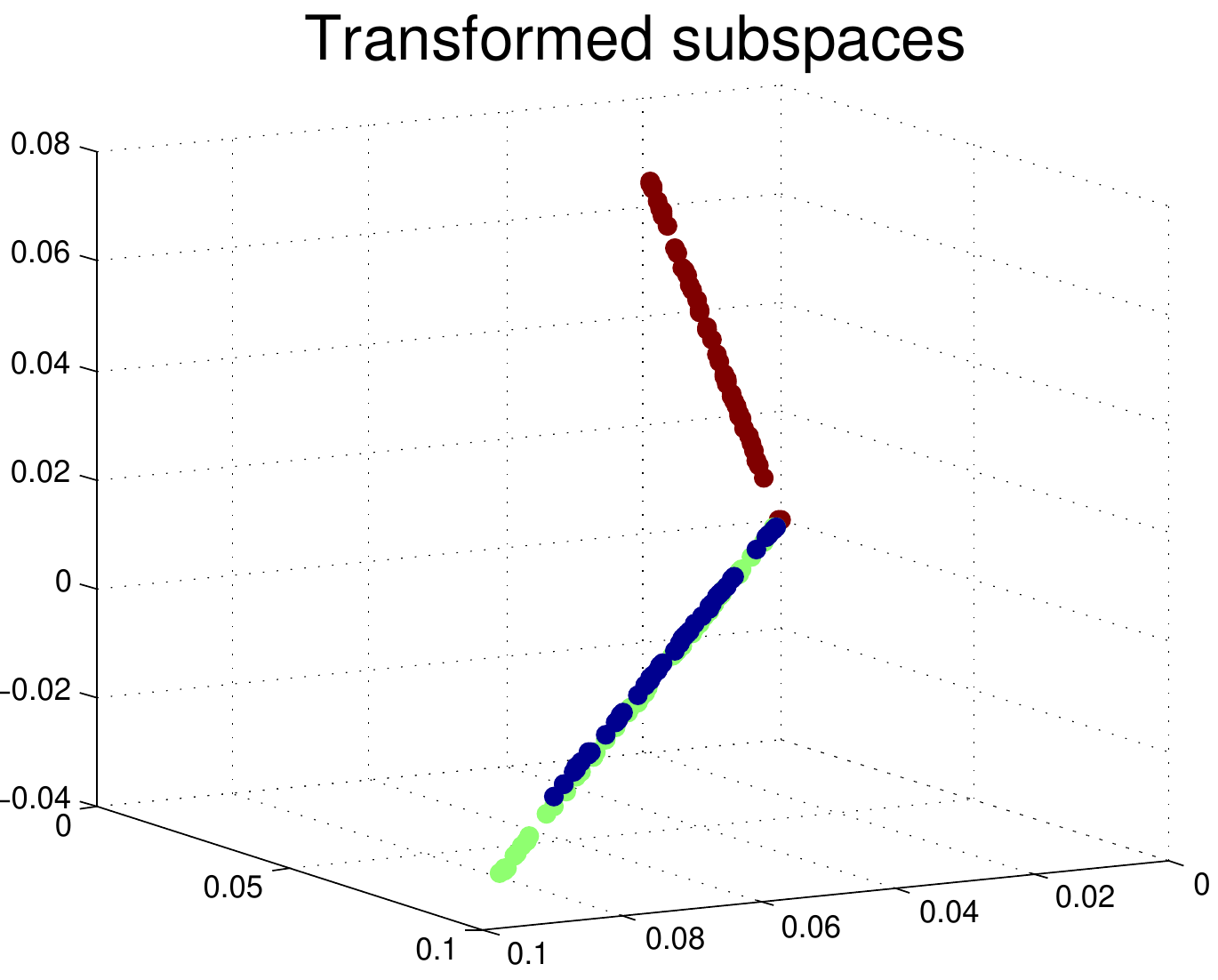}} \\
   \subfloat[][$\begin{bmatrix} \theta_{AB}=1.05,& \theta_{AC}=1.05, &\theta_{AD}=1.05, \\ \theta_{BC}=1.32,& \theta_{BD}=1.39, &\theta_{CD}=0.53 \end{bmatrix}$,\\
 $~~~~\mathbf{Y}_+=\{\mathbf{A}(blue), \mathbf{B}(light~blue)\}$, \\$~~~~\mathbf{Y}_-=\{\mathbf{C}(yellow), \mathbf{D}(red)\} $.] {\label{fig:3d45O} \includegraphics[angle=0, height=0.25\textwidth, width=.47\textwidth]{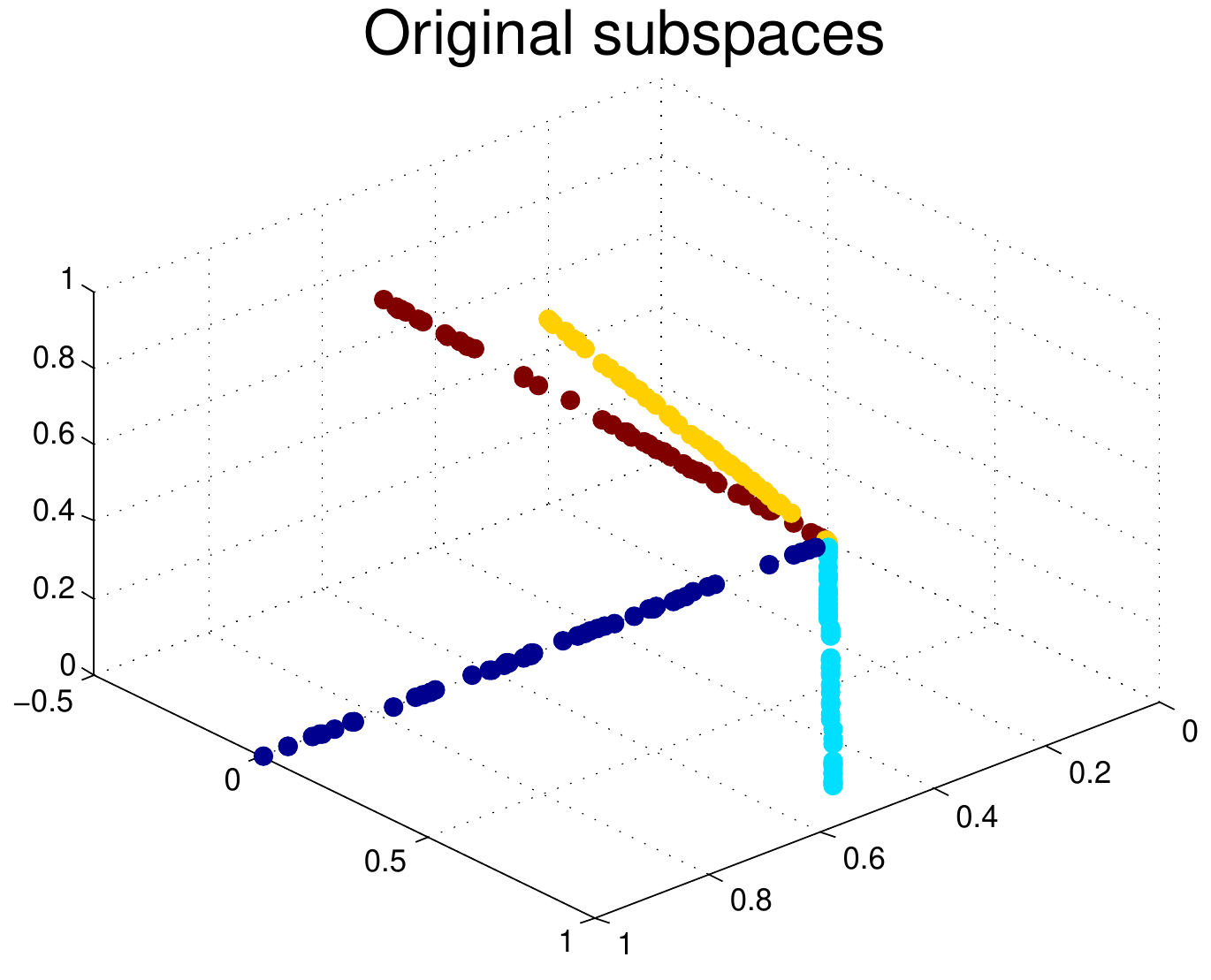}}
  \subfloat[][$\mathbf{T} = \begin{bmatrix} 0.48  & 0.08 &  -0.03 \\ 0.18  & 0.04  &  -0.16 \\ -0.03  &  -0.01  &  0.98 \end{bmatrix}$; \\ $~~~~~\begin{bmatrix} \theta_{AB}=0.03,& \theta_{AC}=1.41, &\theta_{AD}=1.40, \\ \theta_{BC}=1.41,& \theta_{BD}=1.41, &\theta_{CD}=0.01 \end{bmatrix}$.] {\label{fig:3d45T} \includegraphics[angle=0, height=0.25\textwidth, width=.45\textwidth]{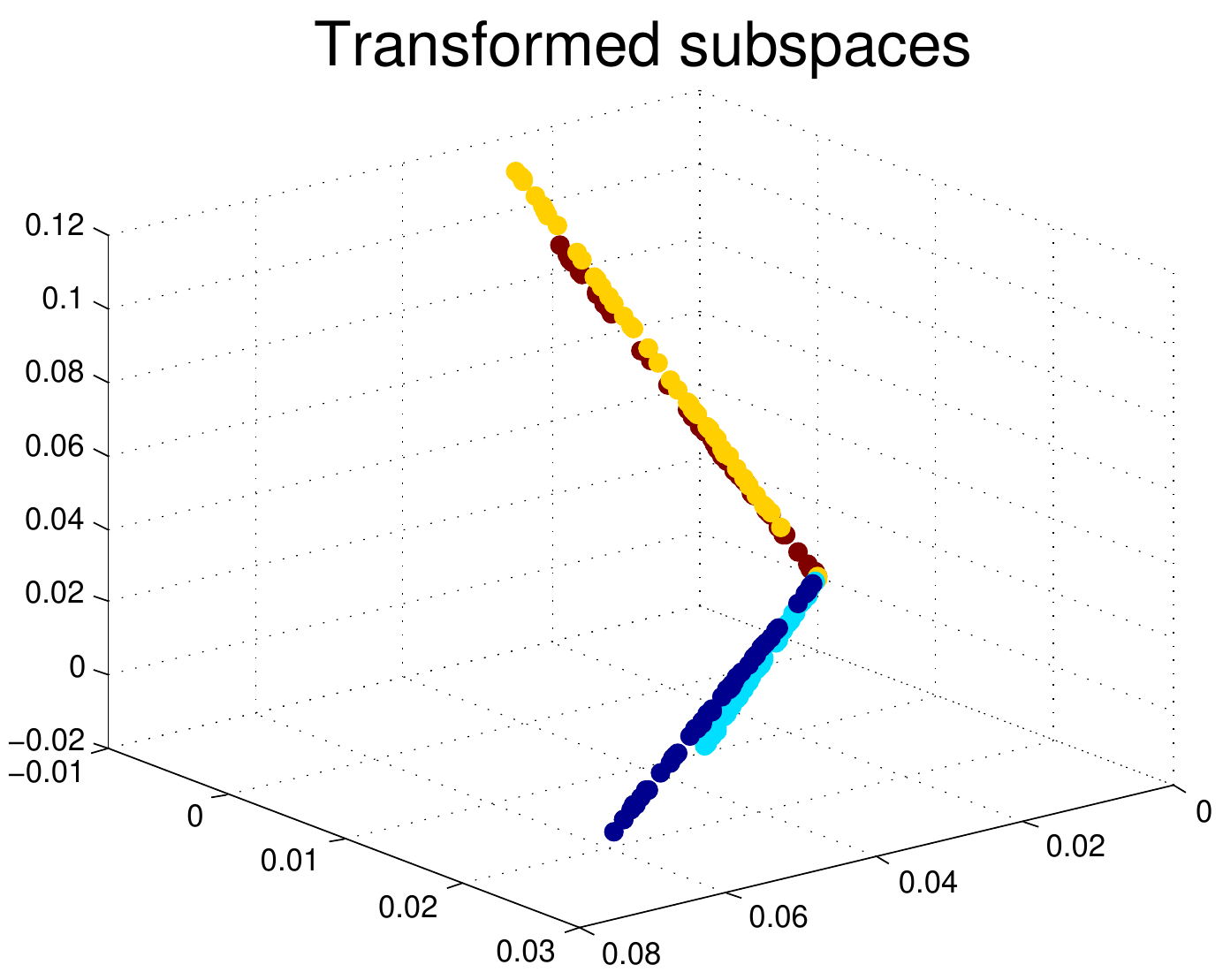}}
\caption{Learning transformation $\mathbf{T}$ using (\ref{nuclear_obj}). We denote the angle between subspaces $\mathbf{A}$ and $\mathbf{B}$ as $\theta_{AB}$ (and analogous for the other pairs of subspaces).
As indicated in (a) and (c), we assign subspaces to different classes $\mathbf{Y}_+$ and $\mathbf{Y}_-$.
Using (\ref{nuclear_obj}), we transform subspaces in (a),(c) to (b),(d) respectively.
We observe that the learned transformation $\mathbf{T}$ increases the inter-class subspace angle towards the maximum $\frac{\pi}{2}$, and reduces intra-class subspace angle towards the minimum $0$.
}
\label{fig:synthetic}
\end{figure*}

One fundamental factor that affects the performance of weak learners in a classification tree is the separation between different class subspaces.
An important notion to quantify the separation between two subspaces $\mathcal{S}_i$ and $\mathcal{S}_j$ is the smallest principal angle $\theta_{ij}$ \cite{prin_angle}, \cite{SSC}, defined as
\begin{align} \label{pangle}
\theta_{ij} = \underset{\mathbf{u} \in \mathcal{S}_i, \mathbf{v} \in \mathcal{S}_j} \min \arccos
\frac{\mathbf{u}'\mathbf{v}}{||\mathbf{u}||_2||\mathbf{v}||_2}.
\end{align}
Note that $\theta_{ij} \in [0, \frac{\pi}{2}].$
We show next that
 the learned transformation $\mathbf{T}$ using the objective function (\ref{nuclear_obj})  maximizes the angle between subspaces of different classes, leading to improved data splitting  in a tree node.
 We start by presenting some basic norm relationships for matrices and their corresponding concatenations.

\begin{theorem} \label{nuclear_ineq}
Let $\mathbf{A}$ and $\mathbf{B}$ be matrices of the same row dimensions, and $\mathbf{[A,B]}$ be the concatenation of $\mathbf{A}$ and $\mathbf{B}$,  we have
\begin{align} \nonumber
||\mathbf{[A,B]}||_*  \le ||\mathbf{A}||_* + ||\mathbf{B}||_*,
\end{align}
with equality obtained if the column spaces of $\mathbf{A}$ and $\mathbf{B}$ are orthogonal.
\end{theorem}

\begin{proof}
See Appendix~\ref{sec:nuclear_ineq}.
\end{proof}

Based on this result we have that
\begin{align} \label{nuclear_objval}
||\mathbf{T Y}_+||_* + ||\mathbf{T Y}_-||_* - ||\mathbf{T [Y_+, Y_-]}||_* \ge 0,
\end{align}
and the proposed objective function (\ref{nuclear_obj}) reaches the minimum $0$ if the column spaces of
two classes are orthogonal after applying the learned transformation $\mathbf{T}$; or equivalently,
(\ref{nuclear_obj}) reaches the minimum $0$ when the angle between subspaces of two classes is maximized after transformation, i.e., the smallest principal angle between subspaces equals $\frac{\pi}{2}$.

We now discuss the advantages of adopting the nuclear norm in (\ref{nuclear_obj}) over the rank function and other (popular) matrix norms, e.g., the induced 2-norm and the Frobenius norm.
When we replace the nuclear norm in (\ref{nuclear_obj}) with the rank function, the objective function reaches the minimum when subspaces are disjoint, but not necessarily maximally distant.
If we replace the nuclear norm in (\ref{nuclear_obj}) with the  induced 2-norm norm or the Frobenius norm,
as shown in Appendix~\ref{sec:propos},
the objective function is minimized at the trivial solution $\mathbf{T}=0$, which is prevented by the normalization condition $||\mathbf{T}||_2=1$.

Thus, we adopt the nuclear norm in (\ref{nuclear_obj}) for two major advantages that are not so favorable in the rank function or other (popular) matrix norms:
(a) The nuclear norm is the best convex approximation of the rank function \cite{rank-min}, which helps to reduce the variation within classes (first term in (\ref{nuclear_obj}));
(b) The objective function (\ref{nuclear_obj}) is in general optimized when the distance between subspaces of different classes is maximized after transformation, which helps to introduce separations between the classes.

\subsection{Synthetic Examples}

We now illustrate the properties of the above mentioned learned transformation $\mathbf{T}$
using synthetic examples in Fig.~\ref{fig:synthetic} (real-world examples are presented in Section~\ref{sec:expr}). We adopt a simple gradient descent optimization method (though other modern nuclear norm optimization techniques could be considered) to search for the transformation matrix T that minimizes (\ref{nuclear_obj}).
As shown in Fig.~\ref{fig:synthetic}, the learned transformation $\mathbf{T}$ via (\ref{nuclear_obj}) increases the inter-class subspace angle towards the maximum $\frac{\pi}{2}$, and
reduces intra-class subspace angle towards the minimum $0$.

\subsection{Transformation Learner Model for a Classification Tree}
\label{sec:tlm}

During training, at the $i$-th split node, we denote the arriving training samples as $\mathbf{Y}_i^+$ and $\mathbf{Y}_i^-$.
{ When more than two classes are present at a node, we randomly divide classes into two categories. This step is to purposely introduce node randomness to avoid duplicated trees as discussed later.}
 We then learn a transformation matrix $\mathbf{T}_i$ using (\ref{nuclear_obj}), and represent the
subspaces of $\mathbf{T}_i \mathbf{Y}_i^+$ and $\mathbf{T}_i \mathbf{Y}_i^-$ as $\mathbf{D}_i^+$ and $\mathbf{D}_i^-$ respectively.
The weak learner model at the $i$-th split node is now defined as
$\mathbf{\theta}_i=(\mathbf{T}_i, \mathbf{D}_i^+, \mathbf{D}_i^-)$.
{  During both training and testing}, at the $i$-th split node, each arriving sample $\mathbf{y}$ uses $\mathbf{T_i}\mathbf{y}$ as the feature, and is assigned to $\mathbf{D}_i^+$ or $\mathbf{D}_i^-$ that gives the smaller reconstruction error.

Various techniques are available to perform the above evaluation.
In our implementation,  we obtain
$\mathbf{D}_i^+$ and $\mathbf{D}_i^-$ using the K-SVD method \cite{Elad_KSVD}
 and denote a transformation learner as
$\mathbf{\theta}_i=(\mathbf{T}_i, \mathbf{D}_i^+(\mathbf{D}_i^+)^\dag, \mathbf{D}_i^-(\mathbf{D}_i^-)^\dag)$,
where
$\mathbf{D}^{\dagger}=(\mathbf{D}^\mathbf{T}\mathbf{D})^{-1}\mathbf{D}^\mathbf{T}.$
The split evaluation of a test sample $\mathbf{y}$, $|\mathbf{T_i}\mathbf{y}-\mathbf{D}_i^+(\mathbf{D}_i^+)^\dag\mathbf{T_i}\mathbf{y}|$, only involves matrix multiplication, which is of low computational complexity at the testing time.

{
Given a data point $\mathbf{y} \subseteq \mathbb{R}^d$, in this paper, we considered a square linear transformation $\mathbf{T}$ of size $d \times d$. Note that, if we learn a ``fat" linear transformation $\mathbf{T}$ of size $r \times d$, where $(r<d)$, we enable dimension reduction along with transformation to handle very high-dimensional data.
}

\subsection{Randomness in the Model: Transformation Forest}

During the training phase, we introduce randomness into the forests through a combination of random training set sampling and randomized node optimization. We train each classification tree on a different randomly selected training set. As discussed in \cite{RF2001, RFBook}, this reduces possible overfitting
and improves the generalization of classification forests, also significantly reducing the training time.
The randomized node optimization is achieved by randomly dividing classes arriving at each split node into two categories (given more than two classes), to obtain the training sets $\mathbf{Y_+}$ and $\mathbf{Y_-}$.
{ In (\ref{nuclear_obj}), we learn a transformation optimized for a two-class problem. This randomly class dividing strategy reduces a multi-class
problem into a two-class problem at each node for
transformation learning; furthermore, it introduces node
randomness to avoid generating duplicated trees.
}
Note that (\ref{nuclear_obj}) is non-convex and the employed gradient descent method converges to a local minimum.
Initializing the transformation $\mathbf{T}$ with different random matrices might lead to different local minimum solutions.
The identity matrix initialization of $\mathbf{T}$ in this paper leads to excellent performance,
however, understanding the node randomness introduced by adopting different initializations of  $\mathbf{T}$ is the subject of future research.

\section{Experimental Evaluation}
\label{sec:expr}

This section presents experimental evaluations using public datasets: the MNIST handwritten digit dataset, the Extended YaleB face dataset, and the 15-Scenes natural scene dataset.
 The MNIST dataset consists of 8-bit grayscale handwritten digit images of ``0"  through
``9" and 7000 examples for each class.
The Extended YaleB face dataset contains 38 subjects with near frontal pose under 64 lighting conditions (Fig.~\ref{fig:yaledata}).
All the images are resized to $16 \times 16$ { for the MNIST and the Extended YaleB datasets, which gives a 256-dimensional feature}.
The 15-Scenes dataset contains 4485 images falling into 15  natural scene categories (Fig.~\ref{fig:15scene}). The 15 categories include images of living rooms, kitchens, streets,  industrials, etc. We also present results for 3D data from the Kinect datatset in \cite{RF-online}.
{ We first compare many learners in a tree context
for accuracy and testing time; then we compare with
learners that are common for random forests.
}

\subsection{Illustrative Examples}

\begin{figure} [ht]
\centering
\includegraphics[angle=0, height=0.08\textwidth, width=.5\textwidth]{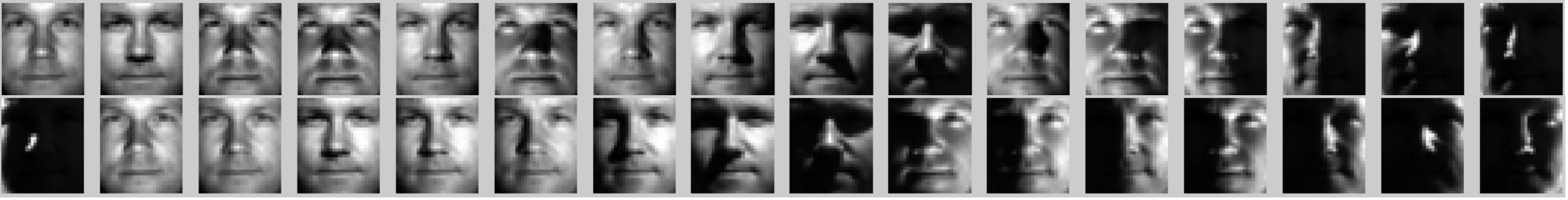}
\caption{Example illumination in the extended YaleB dataset.}
\label{fig:yaledata}
\end{figure}

\begin{figure} [ht]
\centering
\includegraphics[angle=0, height=0.24\textwidth, width=.5\textwidth]{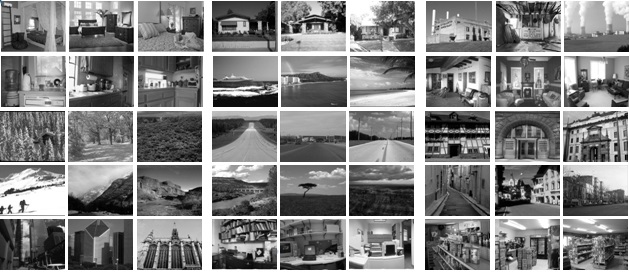}
\caption{15-Scenes natural scene dataset.}
\label{fig:15scene}
\end{figure}

\begin{figure*} [ht]
\centering
 \subfloat[.] {\label{fig:subj38_Om01} \includegraphics[angle=0, height=0.13\textwidth, width=.2\textwidth]{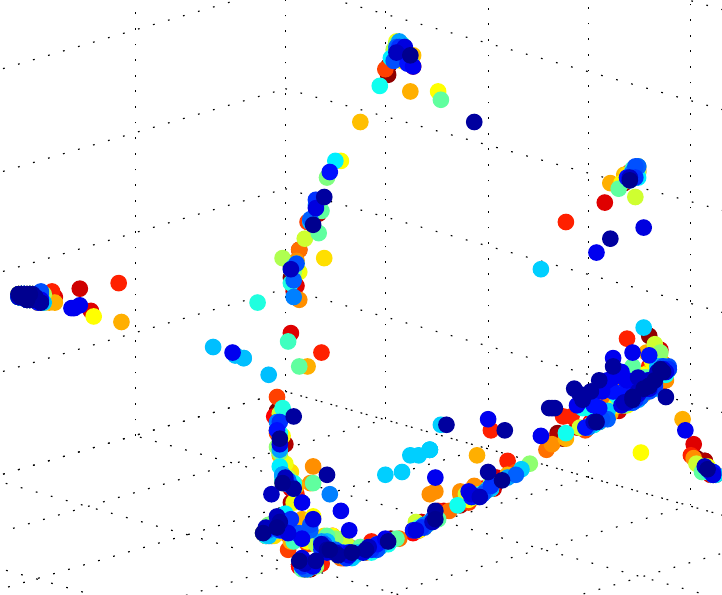} \hspace{0pt}}
  \subfloat[Original.] {\label{fig:subj38_Ob01} \includegraphics[angle=0, height=0.13\textwidth, width=.2\textwidth]{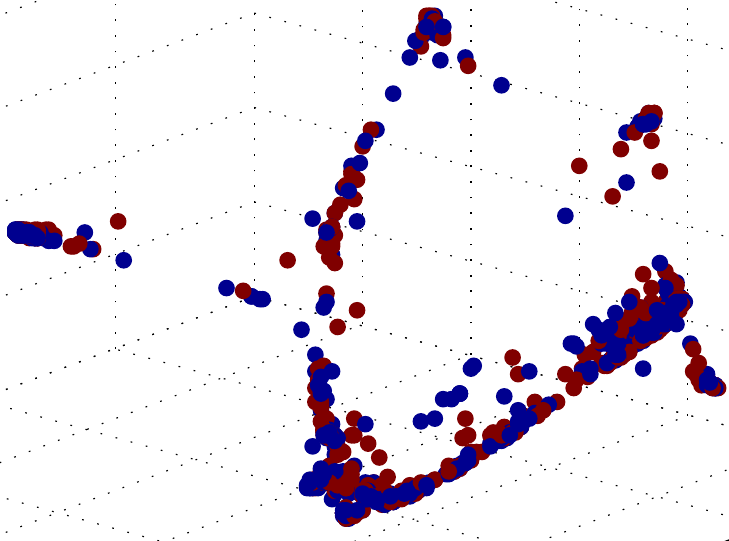}}
  \subfloat[Transformed.] {\label{fig:subj38_Tb01} \includegraphics[angle=0, height=0.13\textwidth, width=.2\textwidth]{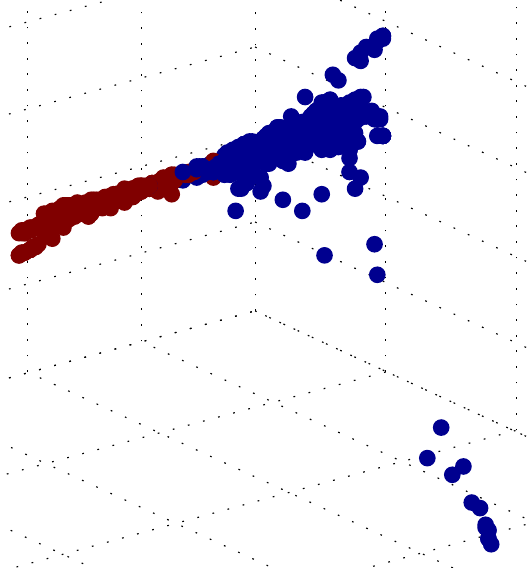}}
    \subfloat[.] {\label{fig:subj38_x01} \includegraphics[angle=0, height=0.13\textwidth, width=.23\textwidth]{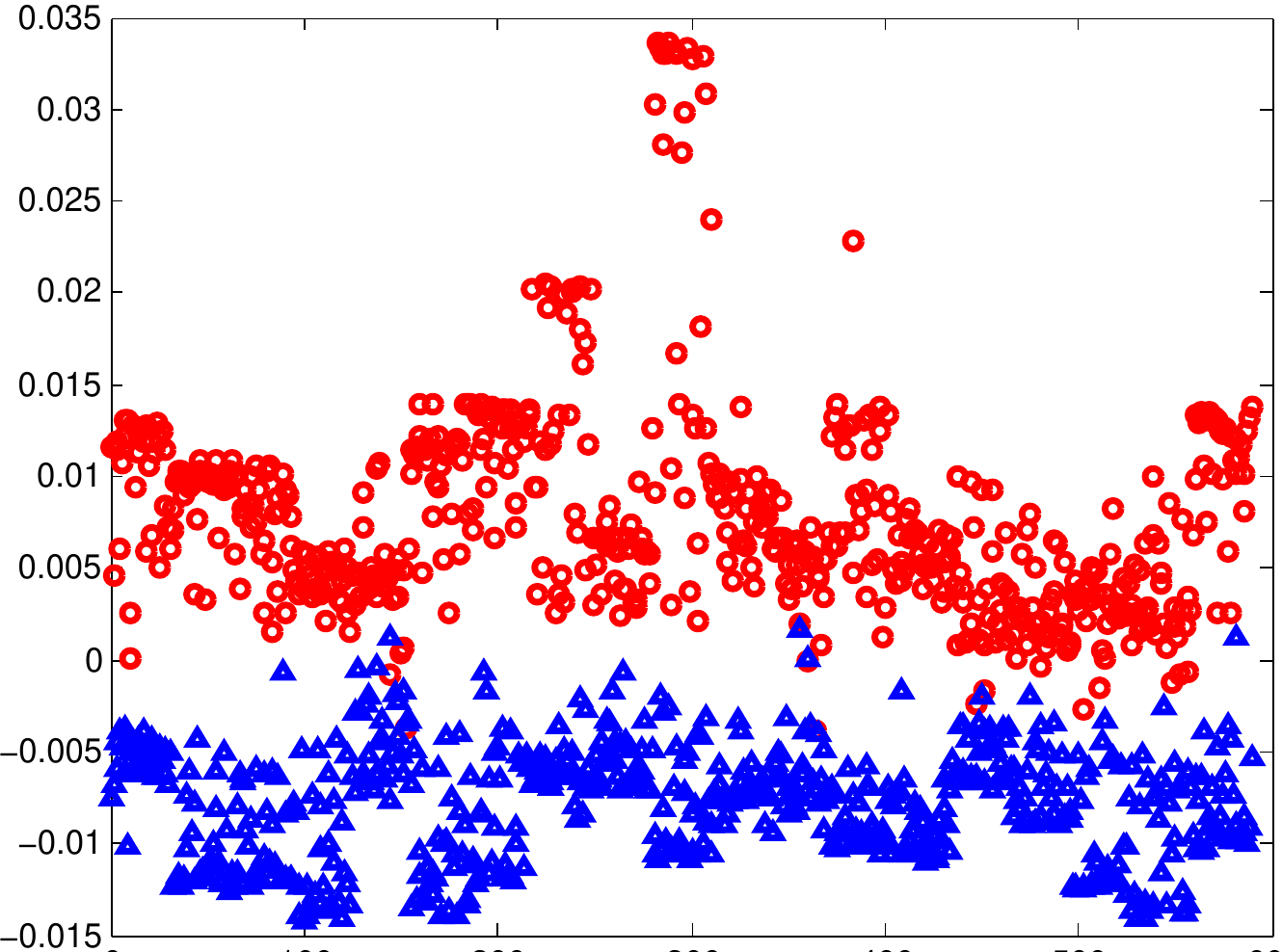}}\\

 \subfloat[.] {\label{fig:subj38_Om02} \includegraphics[angle=0, height=0.13\textwidth, width=.2\textwidth]{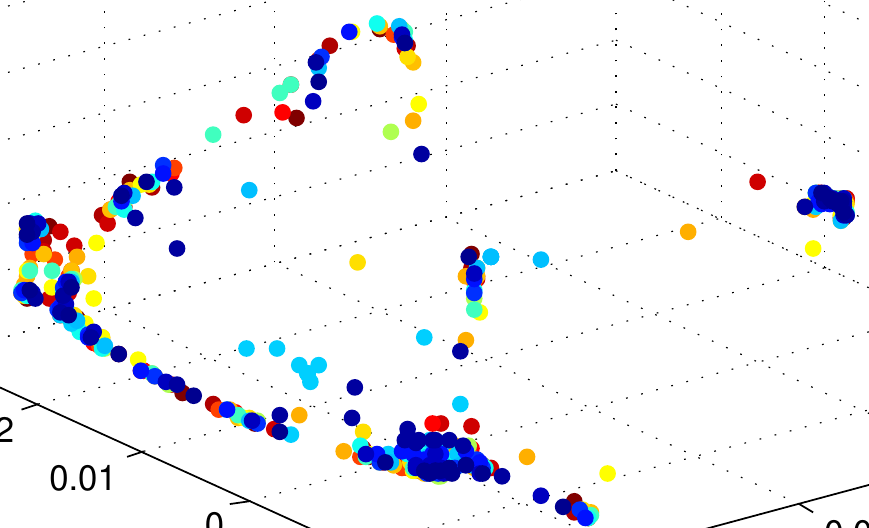} \hspace{0pt}}
  \subfloat[Original.] {\label{fig:subj38_Ob02} \includegraphics[angle=0, height=0.13\textwidth, width=.2\textwidth]{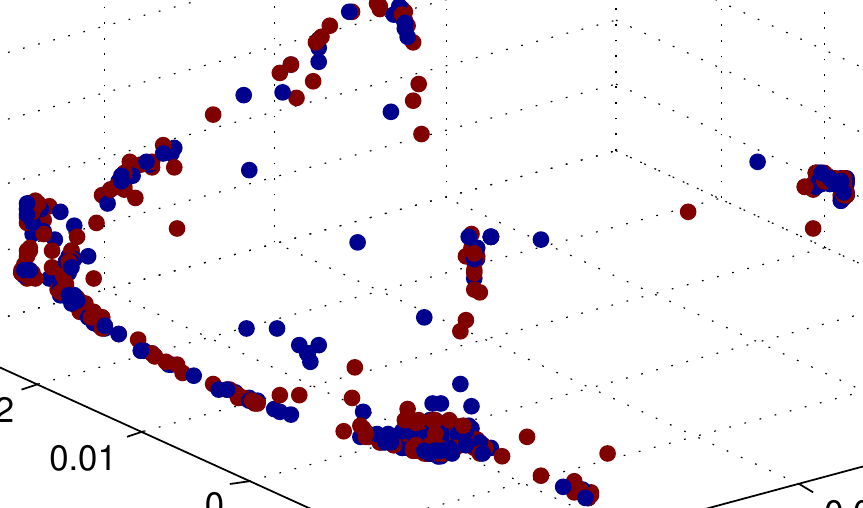}}
  \subfloat[Transformed.] {\label{fig:subj38_Tb02} \includegraphics[angle=0, height=0.13\textwidth, width=.2\textwidth]{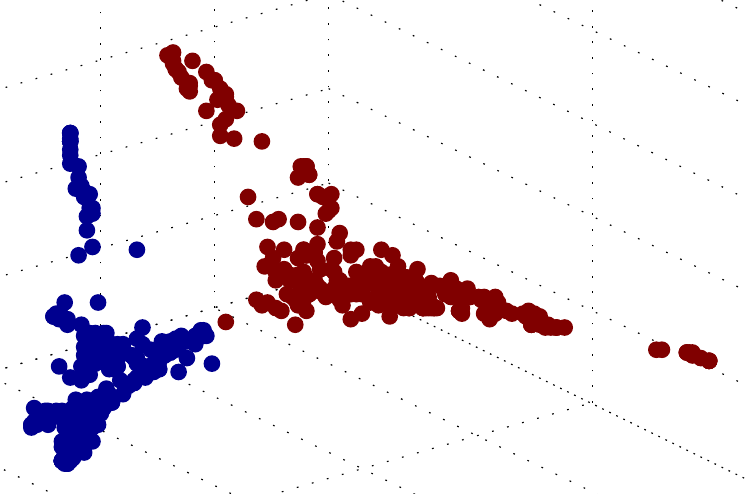}}
    \subfloat[.] {\label{fig:subj38_x02} \includegraphics[angle=0, height=0.13\textwidth, width=.23\textwidth]{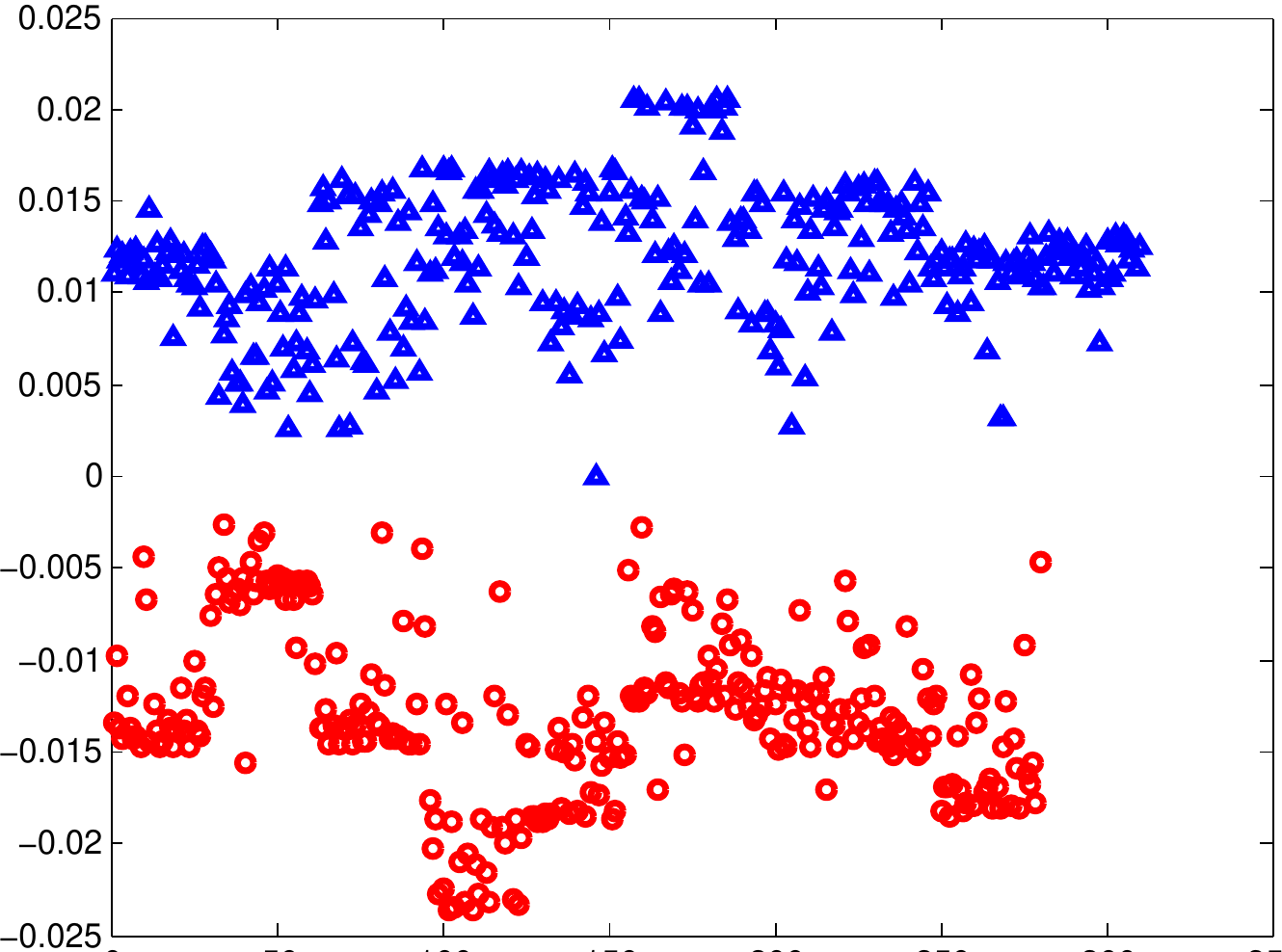}} \\

 \subfloat[.] {\label{fig:subj38_Om03} \includegraphics[angle=0, height=0.13\textwidth, width=.2\textwidth]{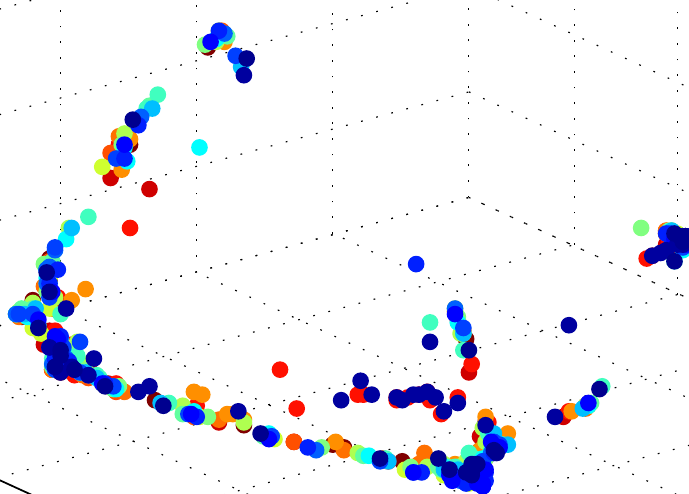} \hspace{0pt}}
  \subfloat[Original.] {\label{fig:subj38_Ob03} \includegraphics[angle=0, height=0.13\textwidth, width=.2\textwidth]{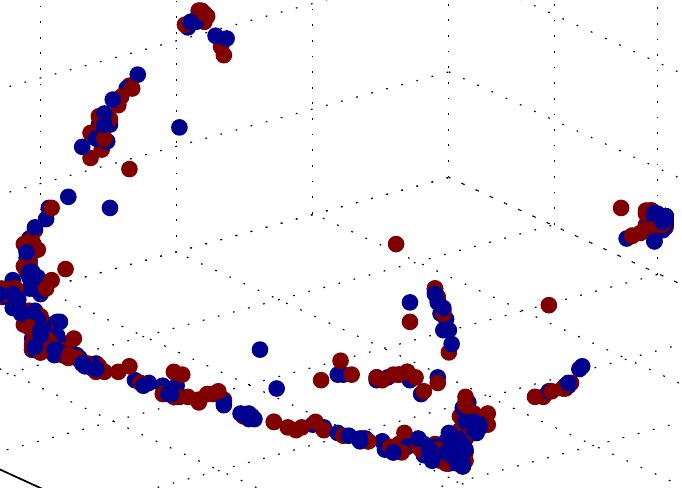}}
  \subfloat[Transformed.] {\label{fig:subj38_Tb03} \includegraphics[angle=0, height=0.13\textwidth, width=.2\textwidth]{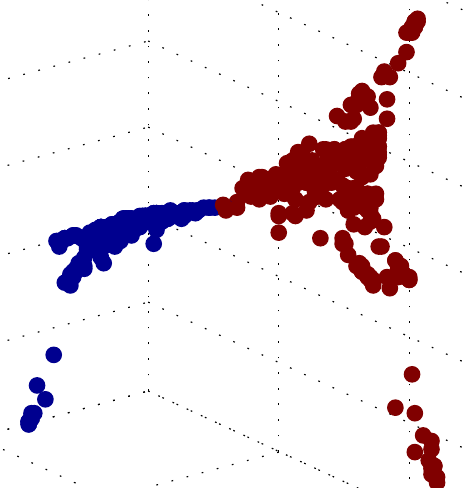}}
    \subfloat[.] {\label{fig:subj38_x03} \includegraphics[angle=0, height=0.13\textwidth, width=.23\textwidth]{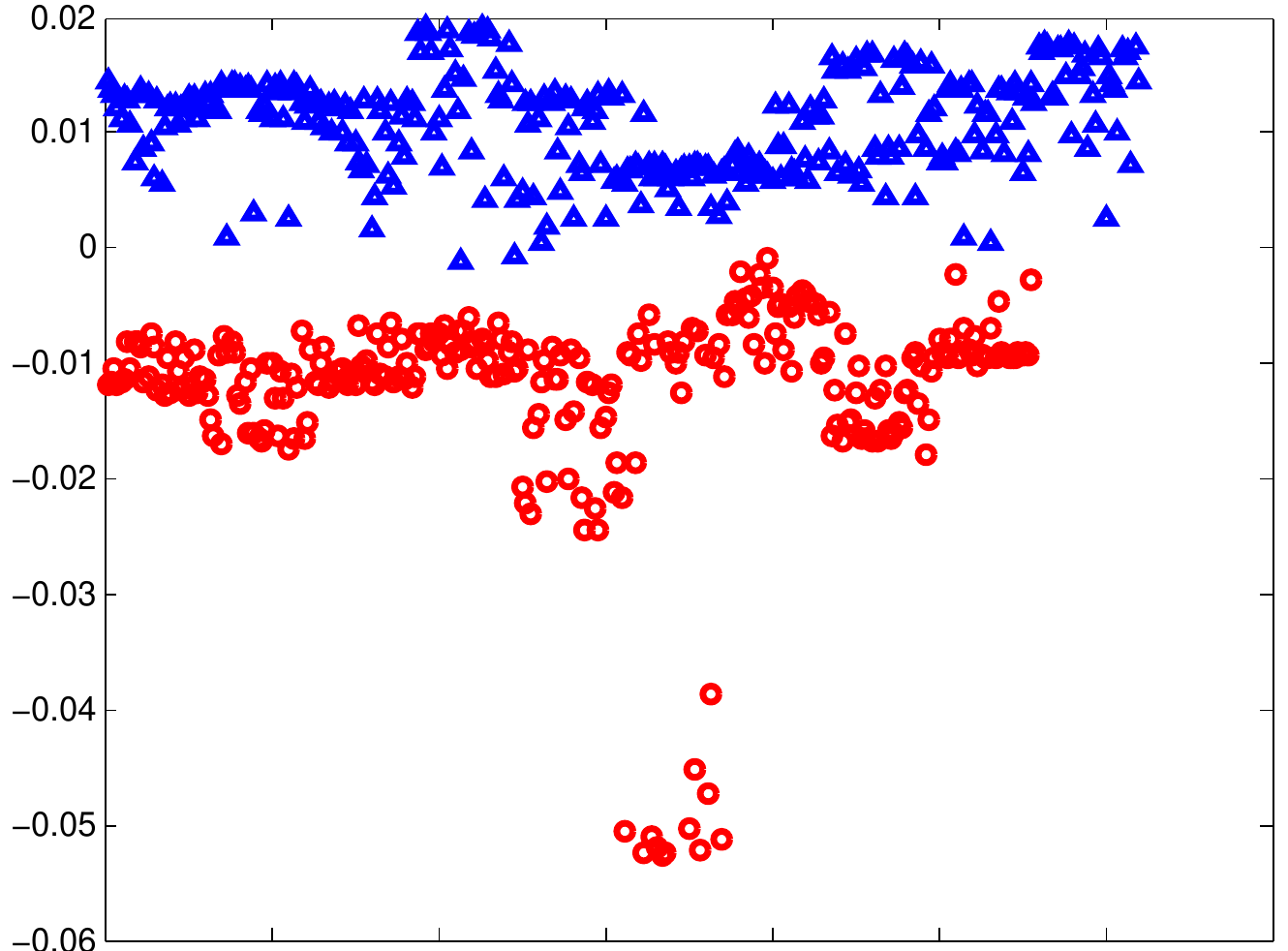}} \\
\caption{Transformation learners in a classification tree constructed on faces of 38 subjects.
The root split node is shown in the first row and its two child nodes are in the 2nd and 3rd rows.
The first column denotes training samples in the original subspaces, with different classes (subjects) in different colors.
 For visualization, the data are plotted with the dimension reduced to 3 using Laplacian Eigenmaps \cite{eigenmap}.
 As shown in the second column, we randomly divide arriving classes into two categories and learn a discriminative transformation using (\ref{nuclear_obj}).
The transformed samples are shown in the third column, clearly demonstrating how data in each class is concentrated while the different classes are separated.
The fourth column shows the first dimension of transformed samples in the third column.
}
\label{fig:38sub}
\end{figure*}

\begin{figure*} [ht]
\centering
 \subfloat[MNIST.] {\label{fig:digit10_acc} \includegraphics[angle=0, height=0.3\textwidth, width=.3\textwidth]{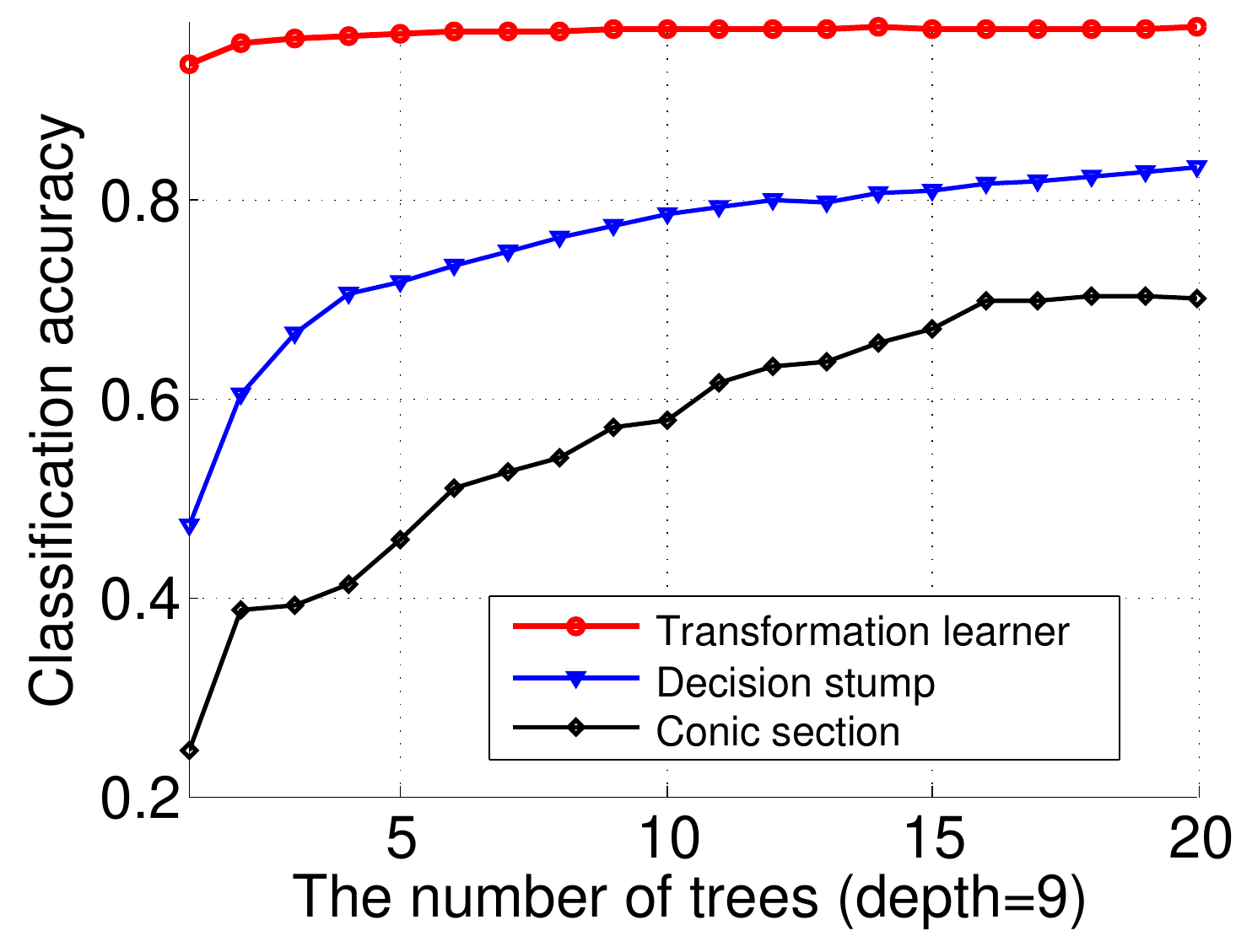} \hspace{0pt}}
  \subfloat[15-Scenes.] {\label{fig:15scene_acc} \includegraphics[angle=0, height=0.3\textwidth, width=.3\textwidth]{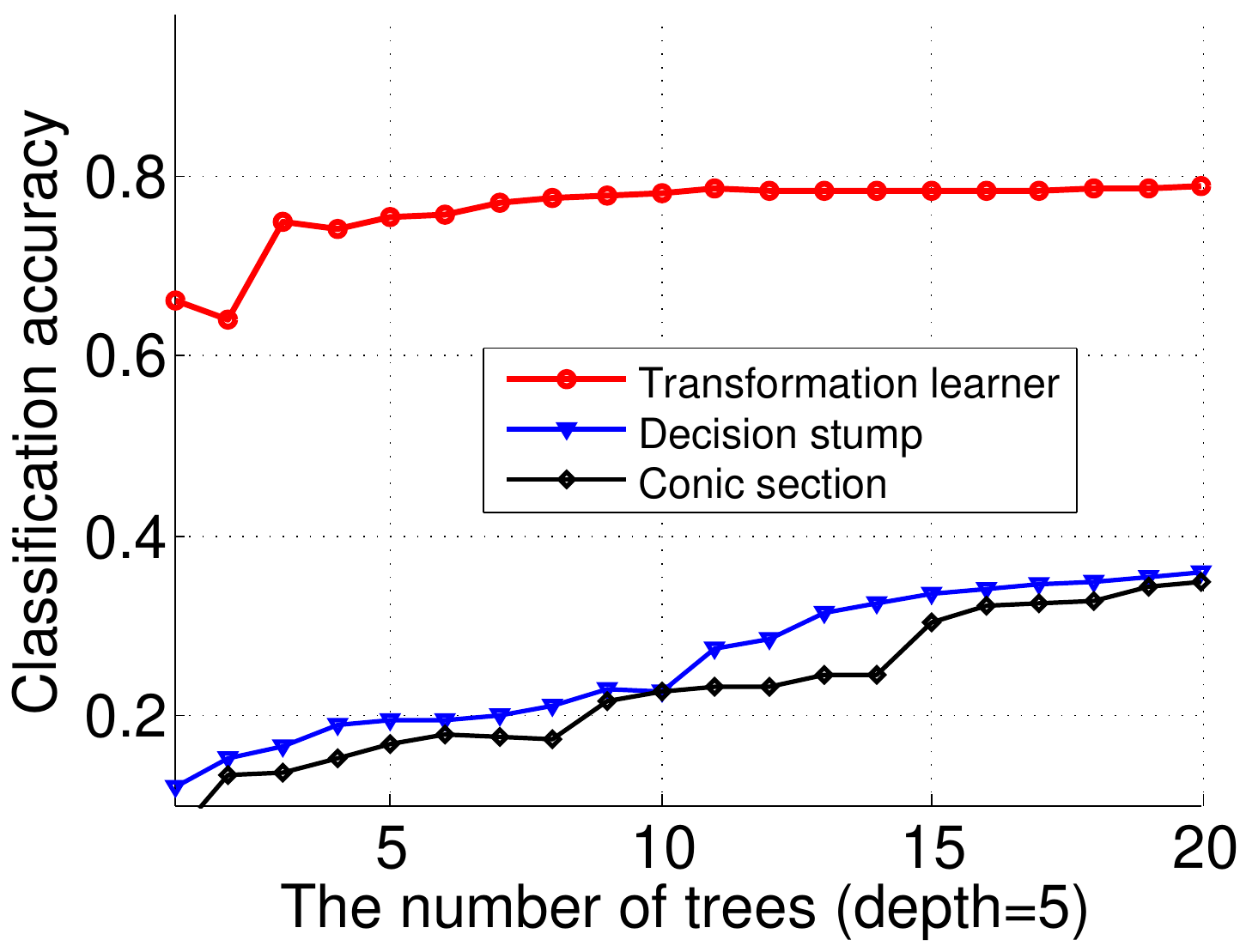}}
    \subfloat[Kinect.] {\label{fig:kinect_acc} \includegraphics[angle=0, height=0.3\textwidth, width=.3\textwidth]{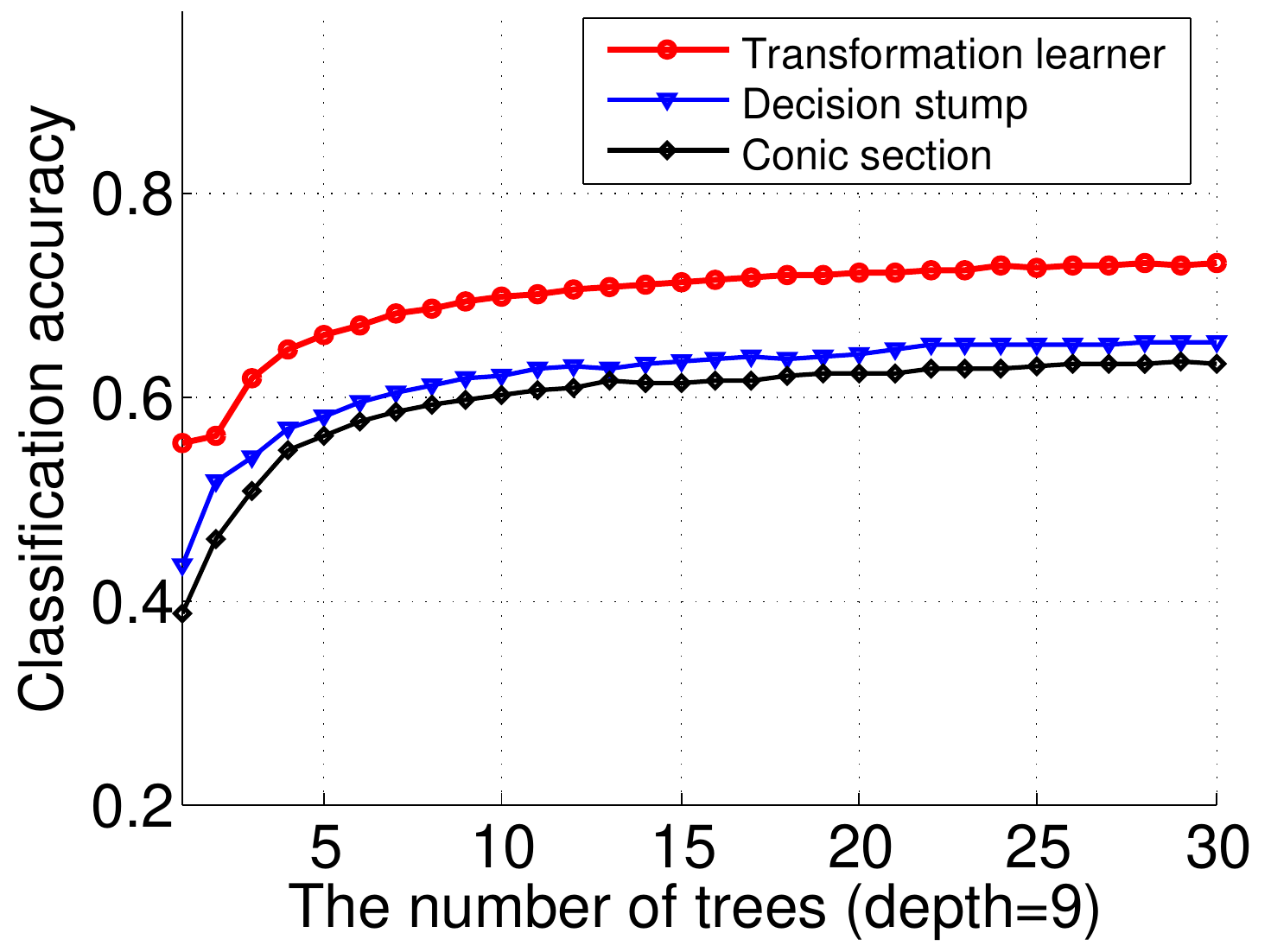}}
\caption{Classification accuracy using transformation learner forests.
}
\label{fig:treeacc}
\end{figure*}

{ We construct classification trees on the extended YaleB face dataset to compare different learners.}
 We split the dataset into two halves by randomly selecting 32 lighting conditions for training, and the other half for testing.
  Fig.~\ref{fig:38sub} illustrates the proposed transformation learner model in a classification tree constructed on faces of all 38 subjects.
The third column shows that transformation learners at each split node enforce separation between two randomly selected categories, and clearly demonstrates how data in each class is concentrated while the different classes are separated.

\begin{table}[ht]
\centering
	\caption{ Classification accuracies (\%) and testing time for the Extended YaleB dataset using classification trees with different learners.}
{\small
	\begin{tabular}{|l|l|l|}
	\hline
Method & Accuracy & { Testing } \\
 & (\%) & { time (s) }\\
	\hline
 \hline
 Non-tree based methods & & \\
  \hline
D-KSVD \cite{Zhang10} & 94.10 &- \\
LC-KSVD \cite{lcksvd} & 96.70 &-\\
SRC \cite{Wright09} & 97.20 &-\\
\hline
\hline
Classification trees & &\\
  \hline
Decision stump (1 tree) & 28.37 & 0.09\\
Decision stump (100 trees)& 91.77 & 13.62\\
Conic section (1 tree) & 8.55 & 0.05\\
Conic section (100 trees)  & 78.20 & 5.04\\
C4.5 (1 tree) \cite{c4.5} & 39.14 & 0.21\\
LDA  (1 tree) & 38.32 & 0.12\\
LDA  (100 trees) & 94.98 & 7.01\\
SVM (1 tree) & 95.23 & 1.62\\
Identity learner (1 tree)& 84.95 & 0.29\\
Transformation learner (1 tree)& \textbf{98.77} & 0.15\\
\hline
	\end{tabular}
}	
	\label{tab:yaleacc}
\end{table}

{  A maximum tree depth is typically specified for random forests to limit the size of a tree \cite{RFBook}, which is different from algorithms like C4.5 \cite{c4.5} that grow the tree only relying on termination criterion. The tree depth in this paper is the maximum tree depth. To avoid
under/over-fitting, we choose the maximum tree depth through a validation process. We also
implement additional termination criteria to prevent further training of a branch, e.g., the
number of samples arriving at a node.
}

In Table~\ref{tab:yaleacc}, we construct classification trees with a maximum depth of 9 using different learners ({  no maximum depth is defined for the C4.5 tree.}).  For reference purpose, we also include the performance of several subspace learning methods,
which provide state-of-the-art classification accuracies on this dataset.
Using a single classification tree, the proposed transformation learner already significantly outperforms the popular weak learners \emph{decision stump} and \emph{conic section} \cite{RFBook}, where 100 trees are used (30 tries are adopted here).
We observe that the proposed learner also outperforms more complex split functions SVM and LDA.
The identity learner denotes the proposed framework but replacing the learned transformation with the identity matrix.
Using a single tree, the proposed approach already outperforms state-of-the-art results reported on this dataset. As shown later, with randomness introduced, the performance in general increases further by employing more trees.

{  While our learner has higher complexity compared to weak learners like decision stump, the performance for random forests is judged by the accuracy and test time. Increasing the number of trees (sublinearly) increases accuracy, at the cost of (linearly) increased test time \cite{RFBook}.
 As shown in Table~\ref{tab:yaleacc}, our learner exhibits similar test time as other weaker learners, but with significantly improved accuracy. By increasing the number of trees, other learners may approach our accuracy but at the cost of orders of magnitude more test time.  Thus, the fact that 1-2 orders of magnitude less trees with our learned matrix outperforms standard random forests illustrates the importance of the proposed general transform learning framework. }

\subsection{Randomized Trees}


\begin{figure*} [ht]
\centering
 \subfloat[.] {\label{fig:digit10_Om01} \includegraphics[angle=0, height=0.13\textwidth, width=.2\textwidth]{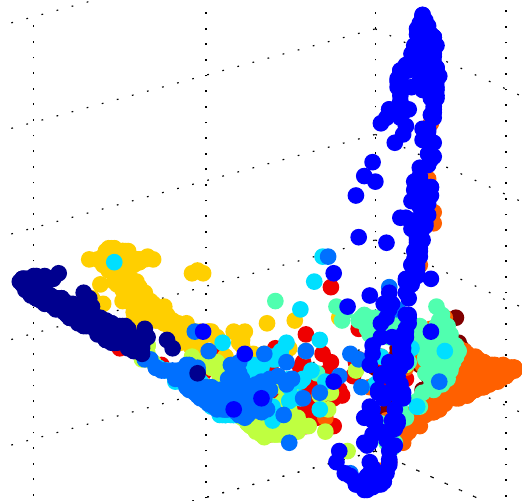} \hspace{0pt}}
  \subfloat[Original.] {\label{fig:digit10_Ob01} \includegraphics[angle=0, height=0.13\textwidth, width=.2\textwidth]{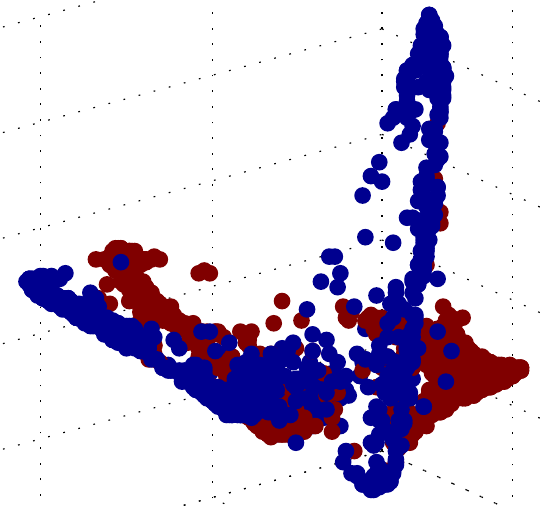}}
  \subfloat[Transformed.] {\label{fig:digit10_Tb01} \includegraphics[angle=0, height=0.13\textwidth, width=.2\textwidth]{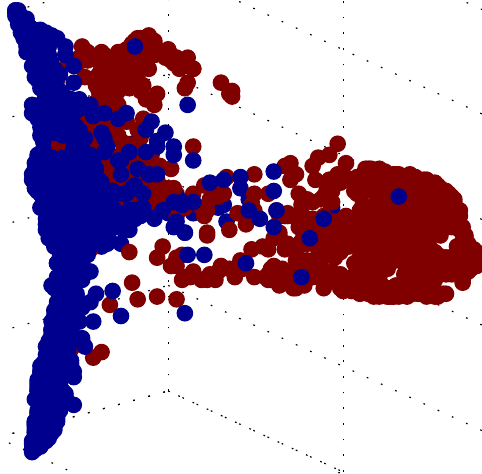}}
    \subfloat[.] {\label{fig:digit10_Tx01} \includegraphics[angle=0, height=0.13\textwidth, width=.23\textwidth]{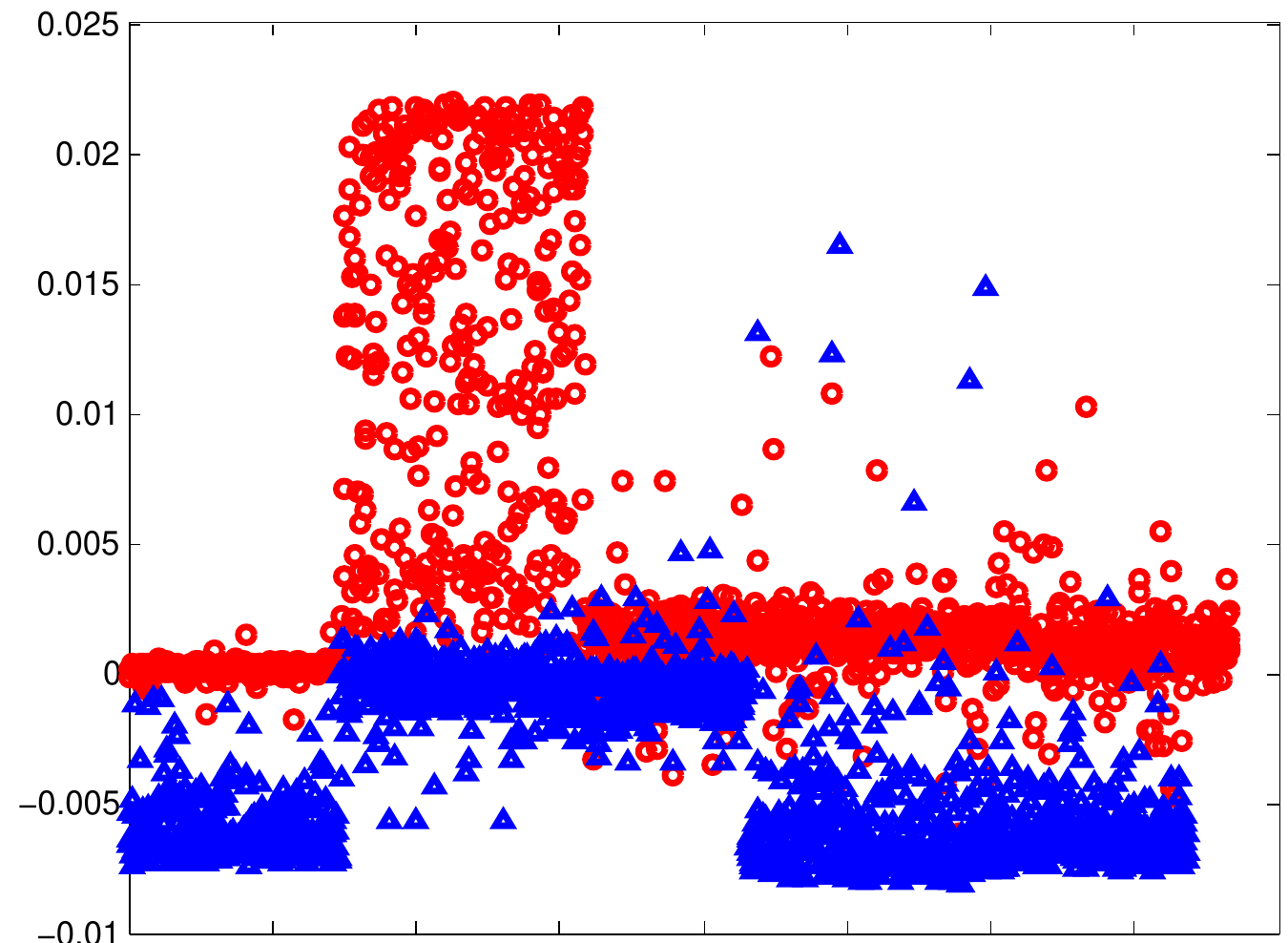}} \\

 \subfloat[.] {\label{fig:digit10_Om02} \includegraphics[angle=0, height=0.13\textwidth, width=.2\textwidth]{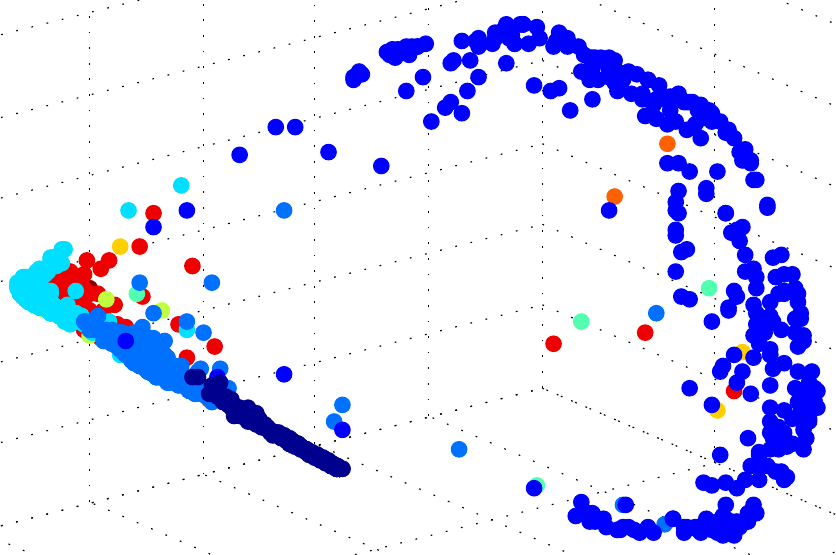} \hspace{0pt}}
  \subfloat[Original.] {\label{fig:subj38_Ob02} \includegraphics[angle=0, height=0.13\textwidth, width=.2\textwidth]{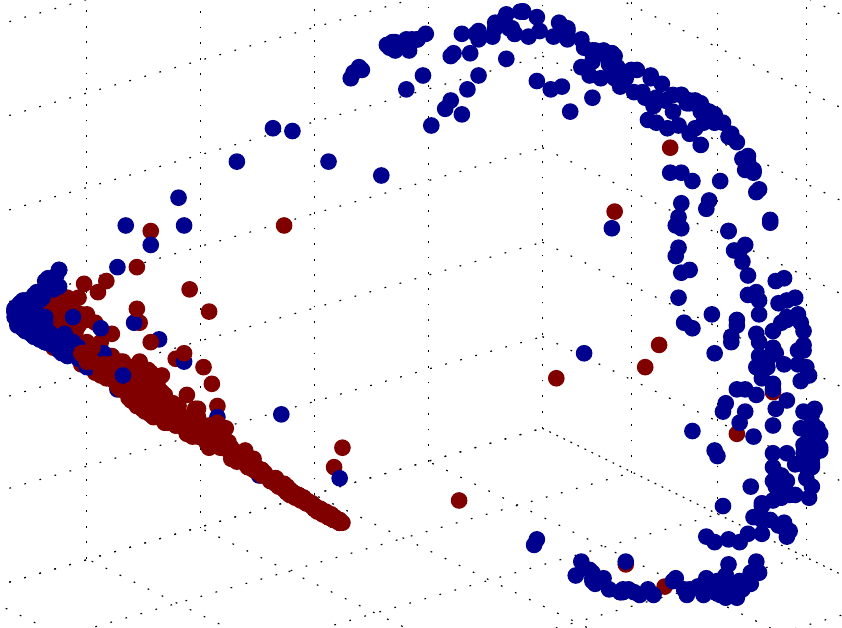}}
  \subfloat[Transformed.] {\label{fig:digit10_Tb02} \includegraphics[angle=0, height=0.13\textwidth, width=.2\textwidth]{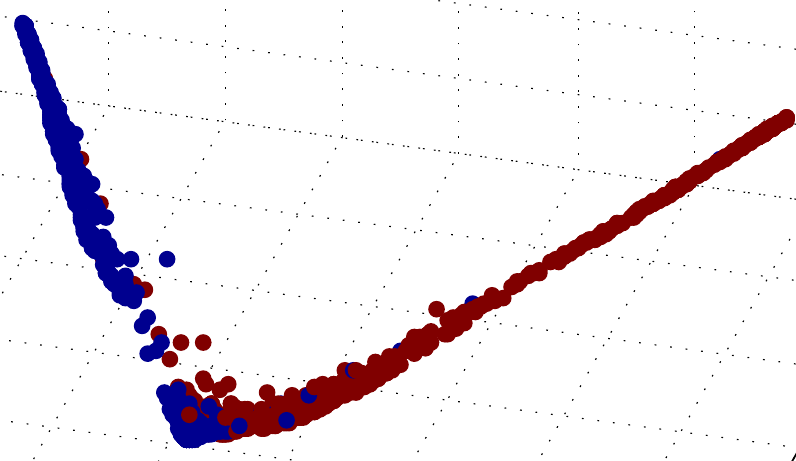}}
    \subfloat[.] {\label{fig:subj38_Tx02} \includegraphics[angle=0, height=0.13\textwidth, width=.23\textwidth]{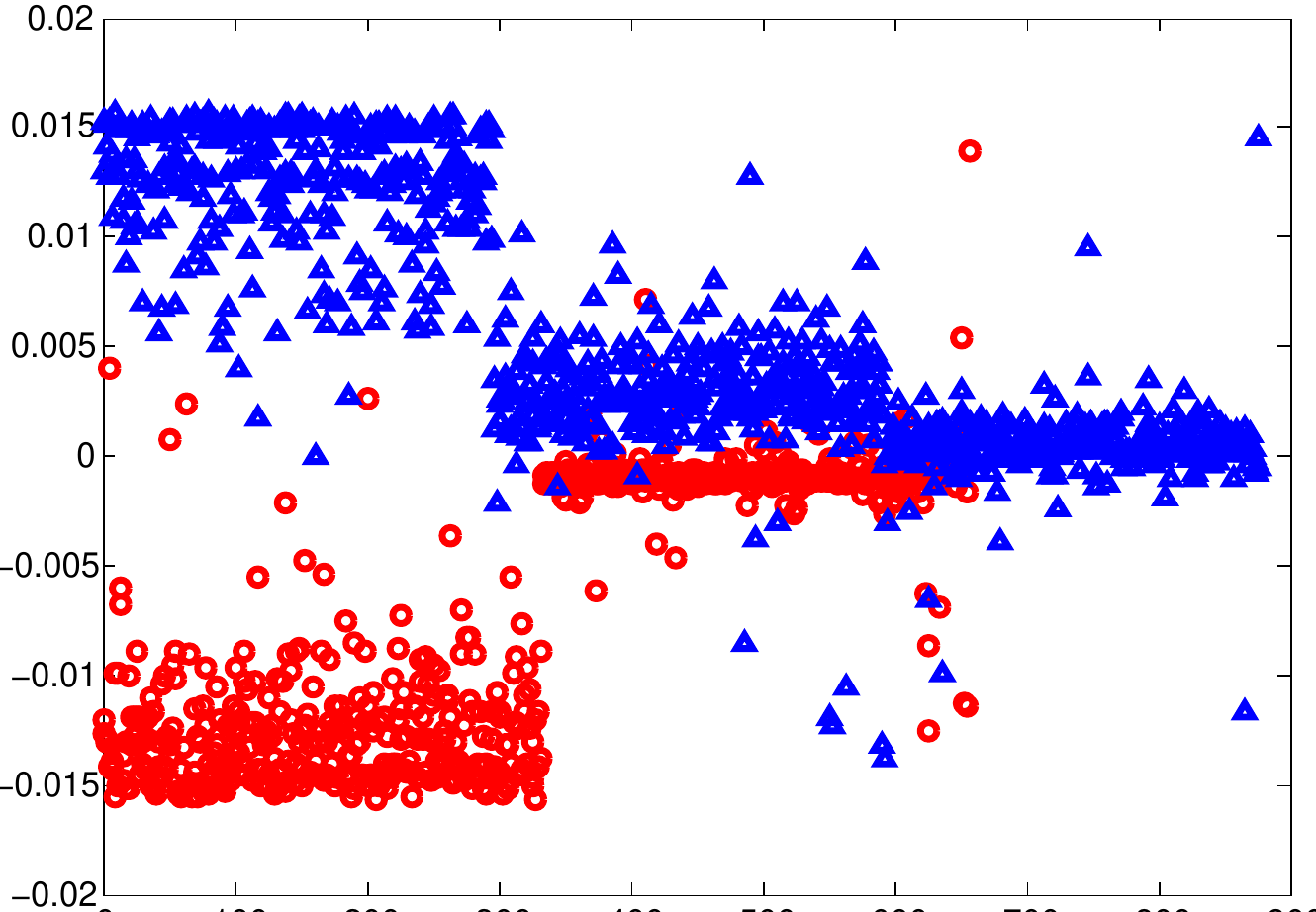}}\\

 \subfloat[.] {\label{fig:digit10_Om03} \includegraphics[angle=0, height=0.13\textwidth, width=.2\textwidth]{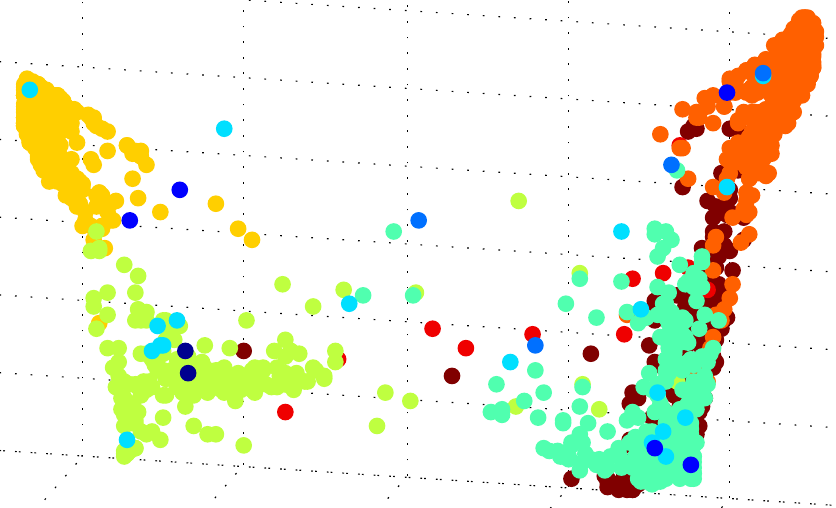} \hspace{0pt}}
  \subfloat[Original.] {\label{fig:digit10_Ob03} \includegraphics[angle=0, height=0.13\textwidth, width=.2\textwidth]{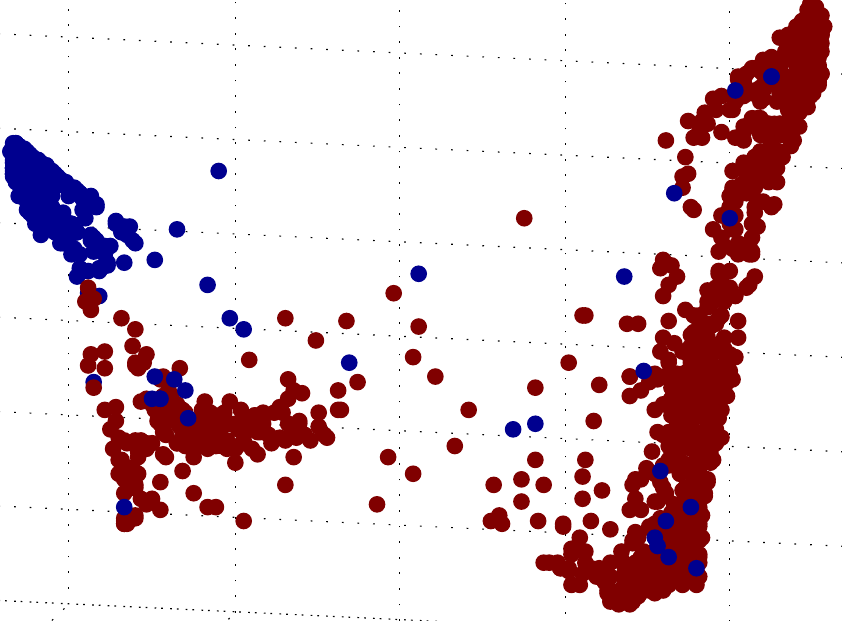}}
  \subfloat[Transformed.] {\label{fig:digit10_Tb03} \includegraphics[angle=0, height=0.13\textwidth, width=.2\textwidth]{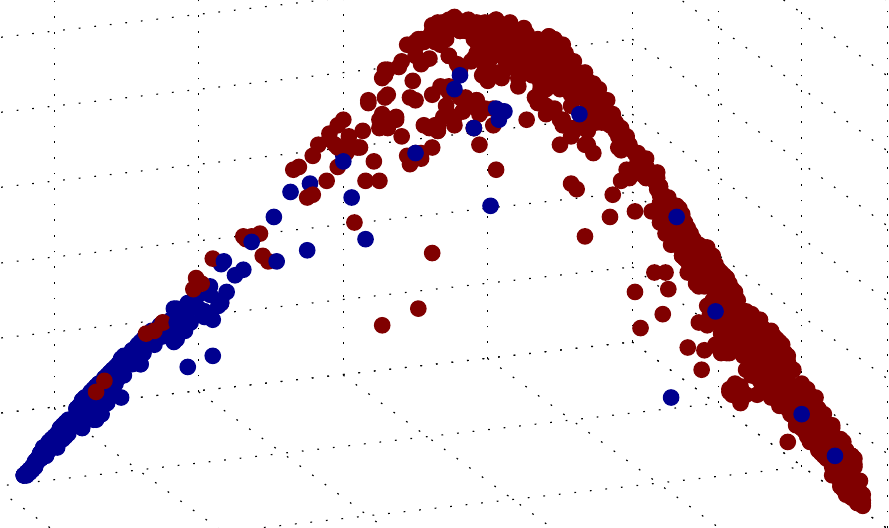}}
    \subfloat[.] {\label{fig:digit10_Tx03} \includegraphics[angle=0, height=0.13\textwidth, width=.23\textwidth]{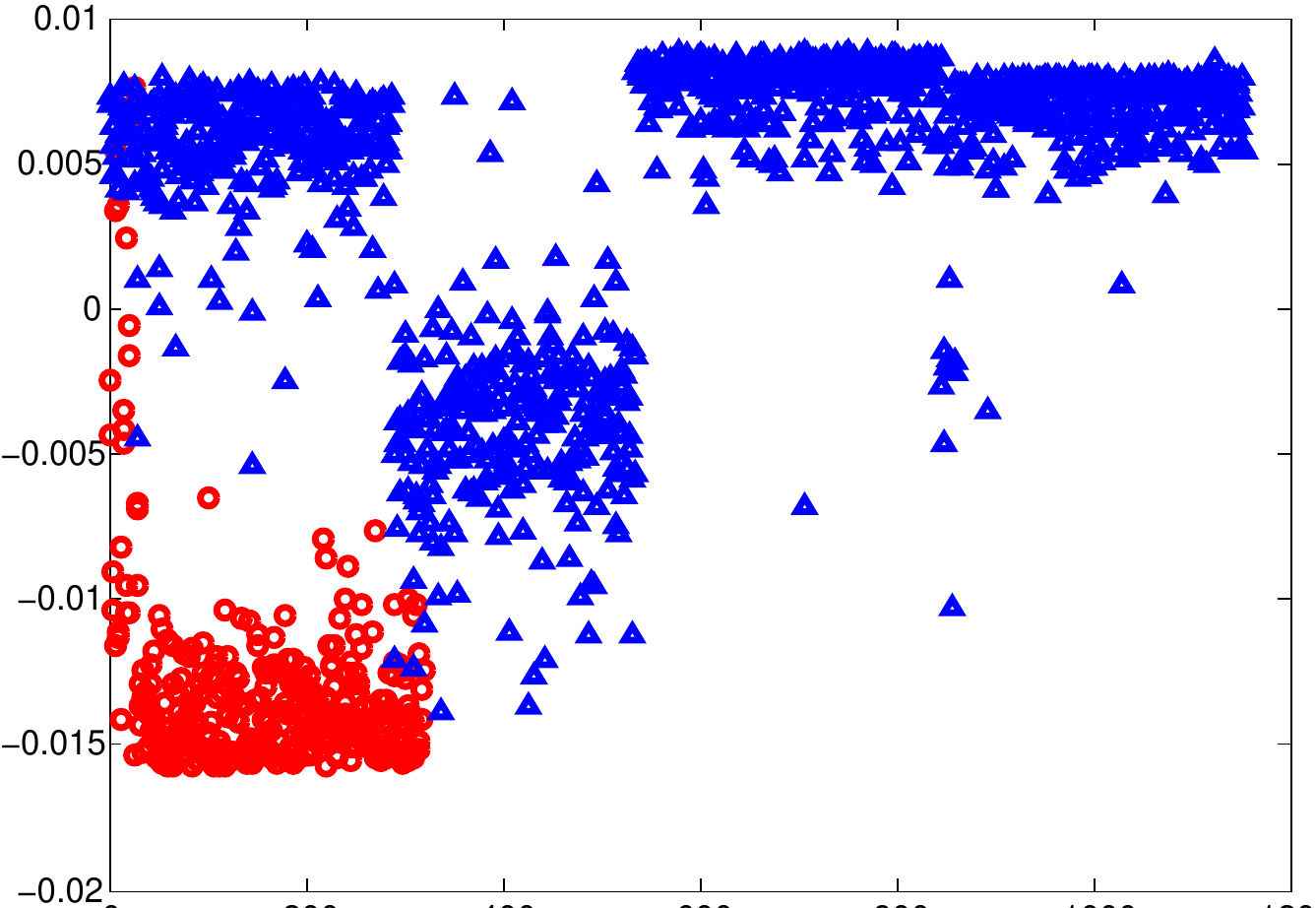}} \\
\caption{Transformation-based learners in a classification tree constructed on the MNIST dataset.
The root split node is shown in the first row and its two child nodes are in the 2nd and 3rd rows.
 The first column denotes training samples in the original subspaces, with different classes in different colors.
 For visualization, the data are plotted with the dimension reduced to 3 using Laplacian Eigenmaps \cite{eigenmap}.
 As shown in the second column, we randomly divide arriving classes into two categories and learn a discriminative transformation using (\ref{nuclear_obj}).
The transformed samples are shown in the third column, clearly demonstrating how data in each class is concentrated while the different classes are separated.
The fourth column shows the first dimension of transformed samples in the third column.
}
\label{fig:digit10}
\end{figure*}

We now evaluate the effect of random training set sampling using the MNIST dataset.
The MNIST dataset has a training set of 60000 examples, and a test set of 10000 examples.
We train 20 classification trees  with a depth of 9, each using only $10\%$ randomly selected training samples (In this paper, we select the random training selection rate to provide each tree about 5000 training samples).
As shown in Fig.~\ref{fig:digit10_acc}, the classification accuracy increases from 93.74\% to 97.30\% by increasing the number of trees to 20. Fig.~\ref{fig:digit10} illustrates in detail the proposed transformation learner model in one of the trees.
As discussed, increasing the number of trees (sublinearly) increases accuracy, at the cost of (linearly) increased test time. Though reporting a better accuracy with hundreds of trees is an option (with limited pedagogical value), a few ($\sim$20) trees are sufficient to illustrate the trade-off between accuracy and performance.

Using the 15-Scenes dataset in Fig.~\ref{fig:15scene}, we further evaluate the effect of randomness introduced by randomly dividing classes arriving at each split node into two categories. We randomly use 100 images per class for training and used the remaining data for testing. We train 20 classification trees with a depth of 5, each using all training samples. As shown in Fig.~\ref{fig:15scene_acc}, the classification accuracy increases from 66.23\% to 79.06\% by increasing the number of trees to 20.
 We notice that, with only 20 trees, the accuracy is already comparable to state-of-the-art results reported on this dataset shown in Table~\ref{tab:sceneacc}. We in general expect the performance increases further by employing more trees.

\begin{table}[ht]
\centering
	\caption{Classification accuracies (\%) for the 15-Scenes dataset.}
{\small
	\begin{tabular}{|l|l|}
	\hline
Method & Accuracy (\%) \\
	\hline
 \hline
ScSPM \cite{yang-cvpr09} & \textbf{80.28} \\
KSPM \cite{scene_15} & 76.73 \\
KC \cite{Gemert-eccv08} & 76.67 \\
LSPM \cite{yang-cvpr09} & 65.32\\
\hline
\hline
Transformation forests & 79.06 \\
\hline
	\end{tabular}
}	
	\label{tab:sceneacc}
\end{table}

\subsection{Microsoft Kinect}

We finally evaluate the proposed transformation learner in the task of  predicting human body part labels from a depth image.
We adopt the Kinect datatset provided in \cite{RF-online}, where pairs of $640 \times 480$ resolution depth and body part images are rendered from the CMU mocap dataset.
The 19 body parts and one background class are represented by 20 unique color identifiers in the body part image. For this experiment, we only use the 500 testing poses from this dataset. We use the first 450 poses for training and remaining 50 poses for testing.
During training, we sample 10 pixels for each body part in each pose and produce 190 data points for each
depth image. Each pixel is represented using depth difference from its 96 neighbors with radius 8, 32 and 64 respectively, { forming a 288-dim descriptor.}
We train 30 classification trees  with a depth of 9, each using $5\%$ randomly selected training samples.
As shown in Fig.~\ref{fig:kinect_acc}, the classification accuracy increases from 55.48\% to 73.12\% by increasing the number of trees to 30. Fig.~\ref{fig:kinectbody} shows an example input depth image, the groud truth body parts, and the prediction using the proposed method.


\section{Conclusion}
\label{sec:con}

We introduced a transformation-based learner model for classification forests.
{
Using the nuclear norm as optimization criteria,  we learn a  transformation at each split node
that reduces variations/noises within the classes, and increases separations between the classes.
The final classification results combines multiple random trees. Thereby we expect the proposed framework
to be very robust to noise.}
We demonstrated the effectiveness of the proposed learner for classification forests,
and provided theoretical support to these experimental results reported for very diverse datasets.

\begin{figure} [ht]
\centering
 \subfloat[Depth.] {\label{fig:dep01_acc} \includegraphics[angle=0, height=0.21\textwidth, width=.15\textwidth]{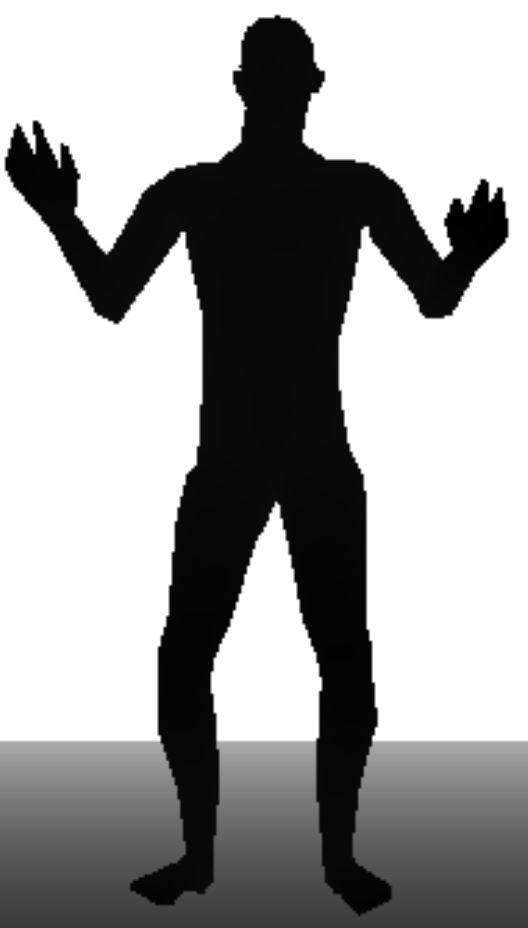} \hspace{0pt}}
  \subfloat[Groundtruth.] {\label{fig:bodygt01_acc} \includegraphics[angle=0, height=0.21\textwidth, width=.15\textwidth]{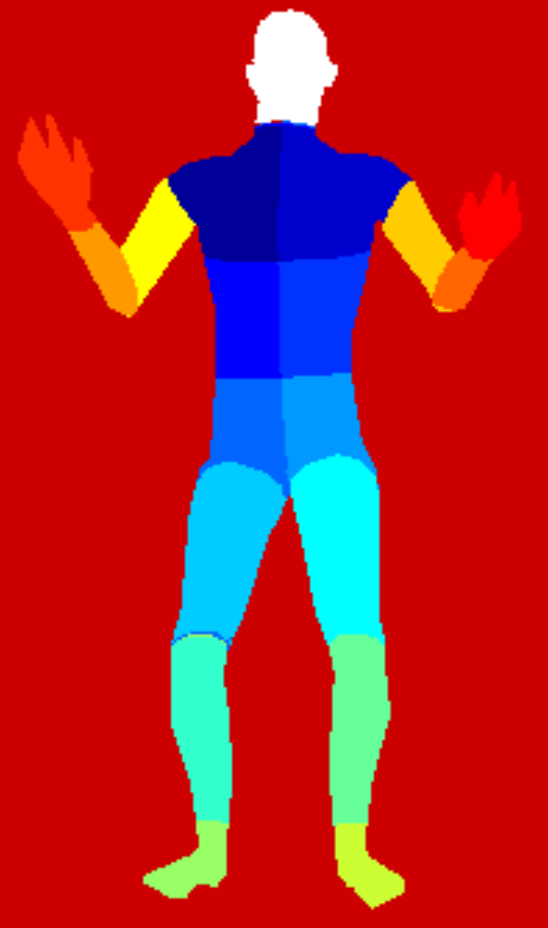}}
    \subfloat[Prediction.] {\label{fig:body01_acc} \includegraphics[angle=0, height=0.21\textwidth, width=.15\textwidth]{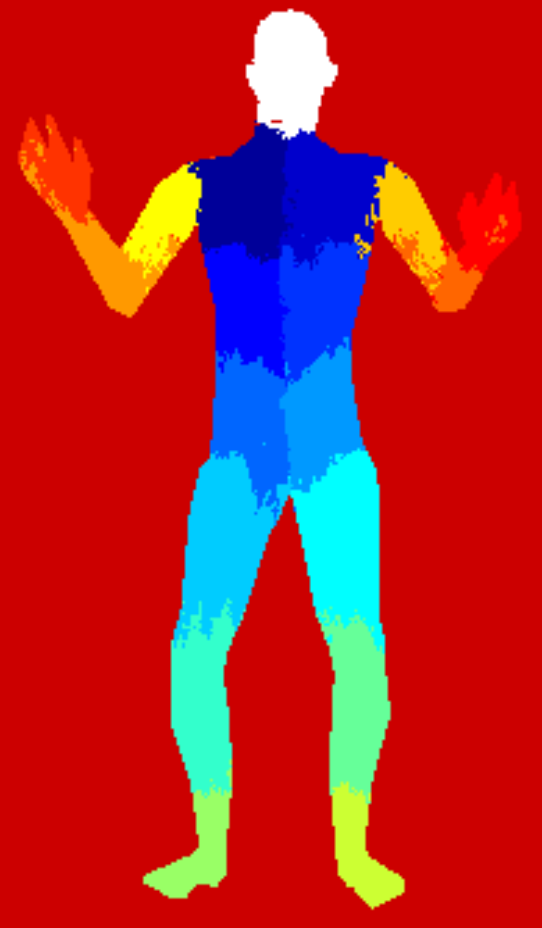}}
\caption{Body parts prediction from a depth image using transformation forests.}
\label{fig:kinectbody}
\end{figure}

\appendix

\section{Proof of Theorem \ref{nuclear_ineq}}
\label{sec:nuclear_ineq}

\begin{proof}
We know that (\cite{nuclear-min})
\begin{align}
||\mathbf{A}||_* =  \underset{\substack{\mathbf{U}, \mathbf{V} \\ \mathbf{A=UV'}}} \min \frac{1}{2}( ||\mathbf{U}||^2_F + ||\mathbf{V}||^2_F). \nonumber
\end{align}
 We denote $\mathbf{U_A}$ and $\mathbf{V_A}$
 the matrices that achieve the minimum; same for $\mathbf{B}$, $\mathbf{U_B}$ and $\mathbf{V_B}$; and same for the concatenation $\mathbf{[A,B]}$, $\mathbf{U_{[A,B]}}$ and $\mathbf{V_{[A,B]}}$.
\noindent We then have
\begin{align} \nonumber
||\mathbf{A}||_*  &= \frac{1}{2}( ||\mathbf{U_A}||^2_F + ||\mathbf{V_A}||^2_F), \\ \nonumber
||\mathbf{B}||_*  &= \frac{1}{2}( ||\mathbf{U_B}||^2_F + ||\mathbf{V_B}||^2_F) .\\ \nonumber
\end{align}
The matrices $[\mathbf{U_A}, \mathbf{U_B}]$ and $[\mathbf{V_A}, \mathbf{V_B}]$ obtained by concatenating the matrices that achieve the minimum for $\mathbf{A}$ and $\mathbf{B}$ when computing their nuclear norm, are not necessarily the ones that achieve the corresponding  minimum in the nuclear norm computation of the concatenation matrix  $[\mathbf{A}, \mathbf{B}]$.
It is easy to show that
\begin{align}
||\mathbf{[A,B]}||_F^2  = ||\mathbf{A}||_F^2 + ||\mathbf{B}||_F^2, \nonumber
\end{align}
where $||\mathbf{\cdot}||_F$ denotes the Frobenius norm. Thus, we have
\begin{align} \nonumber
||\mathbf{[A,B]}||_*  &= \frac{1}{2}( ||\mathbf{U_{[A,B]}}||^2_F + ||\mathbf{V_{[A,B]}}||^2_F) \\ \nonumber
& \le \frac{1}{2}( ||[\mathbf{U_A}, \mathbf{U_B}]||^2_F + ||[\mathbf{V_A},\mathbf{V_B}]||^2_F) \\ \nonumber
&= \frac{1}{2}( ||\mathbf{U_A}||^2_F + ||\mathbf{U_B}||^2_F + ||\mathbf{V_A}||^2_F+ ||\mathbf{V_B}||^2_F) \\ \nonumber
&= \frac{1}{2}( ||\mathbf{U_A}||^2_F + ||\mathbf{V_A}||^2_F) \\ \nonumber
&+ \frac{1}{2}(||\mathbf{U_B}||^2_F+||\mathbf{V_B}||^2_F) \\ \nonumber
&= ||\mathbf{A}||_* + ||\mathbf{B}||_* . \\ \nonumber
\end{align}

We now show the equality condition. We perform the singular value decomposition of $\mathbf{A}$ and $\mathbf{B}$ as
\begin{align}  \nonumber
\mathbf{A} &= [\mathbf{U_{A1}} \mathbf{U_{A2}}] \begin{bmatrix}  \mathbf{\Sigma_A} & 0 \\  0& 0 \end{bmatrix} [\mathbf{V_{A1}} \mathbf{V_{A2}}]', \\ \nonumber
\mathbf{B} &= [\mathbf{U_{B1}} \mathbf{U_{B2}}] \begin{bmatrix}  \mathbf{\Sigma_B} & 0 \\  0& 0 \end{bmatrix} [\mathbf{V_{B1}} \mathbf{V_{B2}}]', \\ \nonumber
\end{align}
where the diagonal entries of $\mathbf{\Sigma_A}$ and $\mathbf{\Sigma_B}$ contain non-zero singular values. We have
\begin{align}  \nonumber
\mathbf{AA'} &= [\mathbf{U_{A1}} \mathbf{U_{A2}}] \begin{bmatrix}  \mathbf{\Sigma_A}^2 & 0 \\  0& 0 \end{bmatrix} [\mathbf{U_{A1}} \mathbf{U_{A2}}]', \\ \nonumber
\mathbf{BB'} &= [\mathbf{U_{B1}} \mathbf{U_{B2}}] \begin{bmatrix}  \mathbf{\Sigma_B}^2 & 0 \\  0& 0 \end{bmatrix} [\mathbf{U_{B1}} \mathbf{U_{B2}}]'. \\ \nonumber
\end{align}

\noindent The column spaces of $\mathbf{A}$ and $\mathbf{B}$ are considered to be orthogonal, i.e., $\mathbf{U_{A1}}'\mathbf{U_{B1}}=0$. The above can be written as
\begin{align}  \nonumber
\mathbf{AA'} &= [\mathbf{U_{A1}} \mathbf{U_{B1}}] \begin{bmatrix}  \mathbf{\Sigma_A}^2 & 0 \\  0& 0 \end{bmatrix} [\mathbf{U_{A1}} \mathbf{U_{B1}}]', \\ \nonumber
\mathbf{BB'} &= [\mathbf{U_{A1}} \mathbf{U_{B1}}] \begin{bmatrix}  0 & 0 \\  0& \mathbf{\Sigma_B}^2 \end{bmatrix} [\mathbf{U_{A1}} \mathbf{U_{B1}}]'. \\ \nonumber
\end{align}

\noindent Then, we have
\begin{align}  \nonumber
\mathbf{[A,B][A,B]'} &= \mathbf{AA'+BB'} \\ \nonumber
&= [\mathbf{U_{A1}} \mathbf{U_{B1}}] \begin{bmatrix}  \mathbf{\Sigma_A}^2 & 0 \\  0& \mathbf{\Sigma_B}^2 \end{bmatrix} [\mathbf{U_{A1}} \mathbf{U_{B1}}]'. \\ \nonumber
\end{align}
The nuclear norm $||\mathbf{A}||_*$ is the sum of the square root of the singular values of $\mathbf{A}\mathbf{A}'$. Thus, $||\mathbf{[A,B]}||_*  = ||\mathbf{A}||_* + ||\mathbf{B}||_*.$
\end{proof}

\section{Basic Propositions}
\label{sec:propos}

\begin{proposition} \label{matrix_ineq}
Let $\mathbf{A}$ and $\mathbf{B}$ be matrices of the same row dimensions, and $\mathbf{[A,B]}$ be the concatenation of $\mathbf{A}$ and $\mathbf{B}$,  we have
\begin{align} \nonumber
||\mathbf{[A,B]}||_2  \le ||\mathbf{A}||_2 + ||\mathbf{B}||_2,
\end{align}
with equality if at least one of the two matrices is zero.
\end{proposition}

\begin{proposition} \label{fnorm_ineq}
Let $\mathbf{A}$ and $\mathbf{B}$ be matrices of the same row dimensions, and $\mathbf{[A,B]}$ be the concatenation of $\mathbf{A}$ and $\mathbf{B}$,  we have
\begin{align} \nonumber
||\mathbf{[A,B]}||_F  \le ||\mathbf{A}||_F + ||\mathbf{B}||_F,
\end{align}
with equality if and only if at least one of the two matrices is zero.
\end{proposition}

\bibliography{forest}
\bibliographystyle{icml2014}

\end{document}